\documentclass[11pt]{article}
\usepackage[letterpaper, margin=1in]{geometry}

\usepackage[utf8]{inputenc} 
\usepackage[T1]{fontenc}    
\usepackage{hyperref}       
\usepackage{url}            
\usepackage{booktabs}       
\usepackage{amsfonts}       
\usepackage{nicefrac}       
\usepackage{microtype}      
\usepackage{xcolor}         
\usepackage{graphicx}

\usepackage{amsmath}
\usepackage{amssymb}
\usepackage{mathtools}
\usepackage{amsthm}

\usepackage[capitalize,noabbrev]{cleveref}
\usepackage{natbib}

\usepackage{dsfont}
\usepackage[mathscr]{euscript}
\usepackage{expl3}
\usepackage{colortbl}
\usepackage{thmtools}
\usepackage{multirow}
\usepackage{accents}
\usepackage{enumitem}
\usepackage[font=small,labelfont=bf]{caption}

\usepackage{pifont}
\usepackage{mathpazo}
\usepackage{MnSymbol}
\DeclareMathAlphabet\mathbb{U}{msb}{m}{n}
\usepackage{xpatch}

\def\Zset{\mathbb{Z}}
\def\Rset{\mathbb{R}}

\DeclareMathOperator*{\E}{\mathbb E}

\DeclareMathOperator{\sign}{sign}
\DeclareMathOperator{\supp}{supp}

\DeclarePairedDelimiter{\abs}{\lvert}{\rvert} 
\DeclarePairedDelimiter{\bracket}{[}{]}
\DeclarePairedDelimiter{\curl}{\{}{\}}
\DeclarePairedDelimiter{\paren}{(}{)}

\ExplSyntaxOn
\tl_const:Nn \c_my_uc_alphabet_tl { ABCDEFGHIJKLMNOPQRSTUVWXYZ }
\tl_const:Nn \c_my_full_alphabet_tl { ABCDEFGHIJKLMNOPQRSTUVWXYZ abcdefghijklmnopqrstuvwxyz }

\tl_map_inline:Nn \c_my_uc_alphabet_tl
 { \cs_gset:cpn { c#1 } { \mathcal{#1} } }

\tl_map_inline:Nn \c_my_uc_alphabet_tl
 { \cs_gset:cpn { s#1 } { \mathscr{#1} } }

\tl_map_inline:Nn \c_my_full_alphabet_tl
 {
  \str_if_eq:nnTF { b#1 } { bf } 
    { } 
    {
      \str_if_eq:nnTF { sf#1 } { sf }
        { } 
        {
          \cs_gset:cpn { b#1 } { \ensuremath{\mathbf{#1}} } 
          \cs_gset:cpn { sf#1 } { \ensuremath{\mathsf{#1}} }
        }
    }
 }
\ExplSyntaxOff

\newcommand{\ov}{\overline}
\newcommand{\wt}{\widetilde}
\newcommand{\e}{\epsilon}
\newcommand{\ignore}[1]{}

\hypersetup{
  breaklinks   = true, 
  colorlinks   = true, 
  urlcolor     = blue, 
  linkcolor    = blue, 
  citecolor    = blue  
}

\setlist{itemsep=2pt, topsep=4pt}
\theoremstyle{plain}
\newtheorem{theorem}{Theorem}[section]
\newtheorem{proposition}[theorem]{Proposition}

\newtheorem{corollary}[theorem]{Corollary}
\theoremstyle{definition}
\newtheorem{definition}[theorem]{Definition}

\theoremstyle{remark}

\title{Mind the Gap: Structure-Aware Consistency\\ in Preference Learning}
\author{
  Mehryar Mohri\\
  Google Research \& CIMS\\
  New York, NY\\
  \texttt{mohri@google.com}
  \and
  Yutao Zhong\\
  Google Research\\
  New York, NY\\
  \texttt{yutaozhong@google.com}
}
\date{}

\begin{document}

\maketitle

\begin{abstract}
  Aligning Large Language Models (LLMs) with human intent, whether through
  explicit reward modeling or direct methods such as DPO, fundamentally relies
  on minimizing a surrogate loss as a proxy for the true pairwise ranking
  objective.  We prove that this reliance is flawed for the standard surrogate
  losses used: for the equicontinuous hypothesis sets characteristic of neural
  networks, \emph{no} standard surrogate provides a meaningful consistency
  guarantee. Minimizing the surrogate loss to zero can leave the true ranking
  error arbitrarily high.  To resolve this, we formulate LLM alignment within a
  margin-shifted ranking framework and derive $\sH$-consistency bounds showing
  that enforcing a confidence margin $\gamma$ is not merely beneficial but
  \emph{necessary} for consistency.  We further introduce Structure-Aware
  $\sH$-consistency and a corresponding objective (SA-DPO) that adapts the
  margin to the semantic distance between responses, preventing instability on
  near-synonymous pairs.  Finally, we analyze the trade-off between the margin
  required for consistency and the model's finite capacity to satisfy it,
  revealing a strict hierarchy of surrogate losses: heavy-tailed surrogates
  (e.g., the Polynomial Hinge family) offer strictly superior consistency
  guarantees for capacity-bounded models compared to the logistic loss used in
  DPO.  Experiments on UltraFeedback and Argilla DPO-Mix-7k confirm that SA-DPO
  consistently outperforms DPO and SimPO, with a 58.5\% head-to-head win-rate in
  downstream generation quality.
\end{abstract}

\vskip .2in
{\small \tableofcontents}
\clearpage

\section{Introduction}
\label{sec:intro}

The alignment of Large Language Models (LLMs) with human intent relies, at its
core, on minimizing a convex surrogate loss over pairwise preference data.  This
is true both in \emph{explicit} reward modeling, where a reward function is
trained on preference pairs before being optimized via reinforcement learning
(e.g., PPO; \citealp{schulman2017proximal,stiennon2020learning}), and in
\emph{implicit} approaches such as Direct Preference Optimization (DPO)
\citep{rafailov2023direct}, which bypass the reward modeling step entirely.  In
every case, the surrogate (e.g., the logistic loss) serves as a proxy for the
true objective: the non-convex, discontinuous 0-1 ranking loss.  This ubiquitous
reliance raises a fundamental theoretical question that remains largely
unanswered for deep networks:
\begin{center}
  \emph{Does minimizing these surrogate losses actually guarantee the
    minimization of the true ranking error?}
\end{center}

In this work, we investigate this question through the lens of
\emph{$\sH$-consistency} \citep*{mao2023cross,mao2024universal}.  We formulate
LLM preference learning as a pairwise ranking problem and derive a series of
results that bridge the gap between learning theory and practical fine-tuning.
First, we identify a fundamental theoretical deficiency in standard approaches.
We demonstrate that for \emph{equicontinuous} hypothesis sets, a property
satisfied by neural networks, standard surrogate minimization yields
\emph{vacuous} consistency guarantees.  Specifically, without explicit
constraints, a model can achieve arbitrarily low surrogate risk while
maintaining a high ranking error, effectively ``cheating'' the objective by
shrinking score differences rather than learning the correct ordering.

To resolve this, we introduce the framework of \emph{Margin-Shifted Surrogates}.
We prove that enforcing a confidence gap $\gamma$ is not merely a heuristic, but
a strict requirement for $\sH$-consistency in the deep learning regime.
However, while a uniform margin restores consistency, it is a blunt instrument.
We show that demanding a large, fixed margin on semantically identical pairs
(synonyms) forces the model to hallucinate differences where none exist,
introducing bias and instability.  To address this, we propose
\emph{Structure-Aware $\sH$-consistency} and a corresponding objective,
\emph{Structure-Aware DPO (SA-DPO)}.  SA-DPO dynamically adapts the margin based
on the semantic distance between responses (see
Figure~\ref{fig:structure_aware}), ensuring the model is strictly penalized for
misranking distinct options while correctly relaxing the constraint for
ambiguous or interchangeable pairs.

\begin{figure}[t]
  \centering
  \includegraphics[width=.66\textwidth]{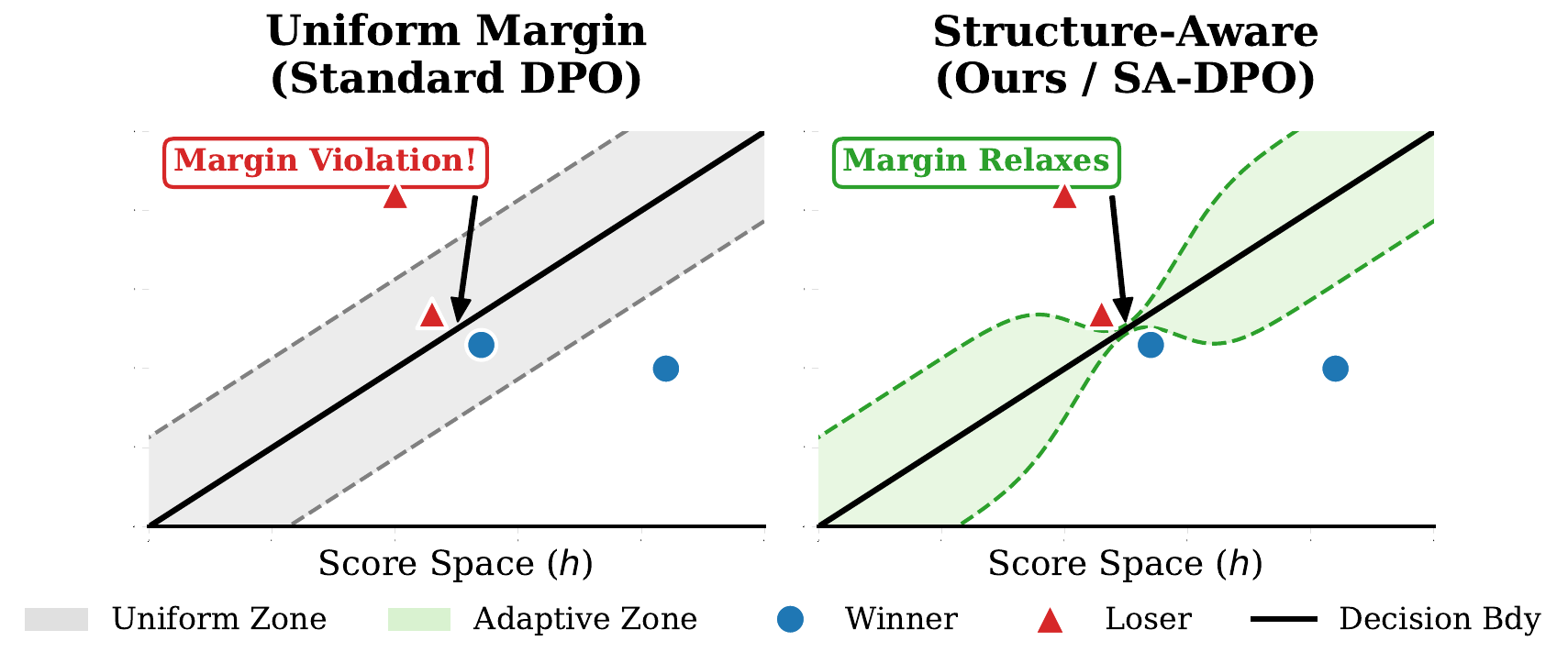}
  \caption{Uniform vs.\ Structure-Aware Margins. (Left) A fixed margin $\gamma$
    forces artificial separation on semantically similar pairs (synonyms),
    causing hallucination. (Right) SA-DPO scales the margin by semantic distance
    $\Delta(y, y')$, relaxing constraints for synonyms while enforcing strict
    separation on distinct pairs.}
  \label{fig:structure_aware}
\end{figure}

Finally, we analyze the cost of enforcing these margins on models with finite
capacity.  We introduce the \emph{Margin-Capacity Profile}, a metric that
quantifies the trade-off between theoretical consistency and a model's ability
to satisfy constraints.  Our analysis reveals a strict hierarchy of loss
functions: while the Logistic loss used in DPO has a linear decay profile,
losses with heavier tails (such as Squared loss in IPO \citep{azar2023general}
or Cubic Hinge loss) offer superior consistency guarantees for capacity-bounded
models.

\paragraph{Contributions.}
\begin{enumerate}
\item We prove that unconstrained surrogate minimization on equicontinuous
  hypothesis sets yields \emph{vacuous} consistency bounds: minimizing the DPO
  loss provides no guarantee that the true ranking error decreases
  (Theorem~\ref{Thm:negative-general}).
\item We derive $\sH$-consistency bounds for margin-shifted surrogates, proving
  that a confidence gap $\gamma$ is \emph{necessary} for consistency
  (Theorem~\ref{th:gamma-shifted-H-consistency}).
\item We introduce SA-DPO, a structure-aware objective that adapts margins to
  the semantic distance between responses, preventing instability on synonyms
  while enforcing strict separation on distinct pairs
  (Theorem~\ref{th:structure-gamma-shifted-H-consistency}).
\item We analyze the Margin-Capacity Profile, establishing a strict hierarchy
  where heavy-tailed losses (Cubic Hinge) strictly outperform light-tailed ones
  (Logistic) for capacity-bounded models.
\end{enumerate}

Our framework also provides a rigorous foundation for recent empirical advances,
most notably SimPO \citep{meng2024simpo} and SLiC \citep{zhao2023slic}.  While
these methods empirically demonstrated the benefit of hard margins, we prove
that such margins are theoretically necessary for consistency.  Furthermore, our
analysis exposes the limitation of their uniform design: enforcing constant
penalties regardless of semantic similarity is suboptimal.  By deriving an
adaptive margin, SA-DPO is the principled, structure-aware evolution of
margin-based alignment.

\paragraph{Organization.}
Section~\ref{sec:related} discusses related work, and
Section~\ref{sec:prelim} introduces our setup and the notion of
$\sH$-consistency bounds. Section~\ref{sec:inconsistency} proves that
unconstrained surrogate minimization over equicontinuous hypothesis sets is
inconsistent. Section~\ref{sec:margin-constrained} shows that hard margin
constraints restore $\sH$-consistency, and Section~\ref{sec:shifted} develops
margin-shifted and structure-aware surrogates, including SA-DPO.
Section~\ref{sec:margin-capacity} analyzes the Margin-Capacity Profile and
Section~\ref{sec:extensions} presents extensions to Bregman-regularized
objectives and the connection with the Bradley-Terry model.
Section~\ref{sec:experiments} reports our experimental results. All proofs are
given in full in the main text, immediately after the corresponding statements.

\section{Related Work}
\label{sec:related}

Our work intersects with three distinct bodies of literature: direct preference
alignment, learning theory for ranking, and structured prediction. We also
discuss connections with social choice and game-theoretic perspectives.

\paragraph{Direct Alignment from Preferences.}
The paradigm of aligning LLMs via preference data was popularized by RLHF
\citep{christiano2017deep, stiennon2020learning}, which involves learning an
explicit reward model. Recently, implicit approaches have gained
prominence. \citet{rafailov2023direct} introduced Direct Preference Optimization
(DPO), deriving the policy update directly from the optimal reward
condition. Subsequent works have proposed variations to address DPO's
limitations: \citet{azar2023general} introduced IPO to prevent over-fitting by
regularizing the reward gap, which we analyze as a margin-shifted squared loss.
\citet{zhao2023slic} proposed SLiC-HF, effectively using a hinge loss on the
probability ratios, which aligns with our analysis of linear-tail margins. More
recently, \citet{meng2024simpo} empirically demonstrated the benefits of
incorporating a hard margin term into the DPO objective (SimPO) and removing the
reference model. Our theory provides a rigorous justification for such
modifications, proving them necessary for consistency rather than just
empirically beneficial. Other non-pairwise approaches like KTO
\citep{ethayarajh2024kto} model utility directly, offering an alternative to the
ranking formulation we study here.
While recent works have extended preference optimization to the online setting
to address exploration and distribution shift
\citep{xiong2024iterative,guo2024direct}, our analysis focuses on the
consistency of the optimization objective itself. Since online methods typically
optimize a surrogate loss in their inner loop, our findings on
margin-consistency remain relevant in iterative settings.

\paragraph{Consistency of Ranking.}
The consistency of surrogate losses is a classical problem in statistical
learning theory. \citet{bartlett2006convexity} provided the foundational
characterization of convex surrogates consistent with the 0-1 binary
classification loss. For the ranking setting, \citet{duchi2010consistency} and
\citet{calauzenes2012calibration} analyzed consistency, highlighting the
difficulty of achieving it without strong domain assumptions. However, these
classical results typically assume universal function approximation, that is,
that the hypothesis space is the set of all measurable functions
($\sH_{\rm{all}}$). In deep learning tasks, $\sH$ is restricted (e.g., to a
family of neural networks), rendering these universal guarantees
inapplicable. Our analysis builds upon the framework of
\emph{$\sH$-consistency} developed by \citet*{awasthi2022h,awasthi2022multi} and
recently extended by \citet{mao2023cross,
  MaoMohriZhong2023ranking,MaoMohriMohriZhong2023twostage,
  MaoMohriZhong2023characterization,MaoMohriZhong2023rankingabs,
  MaoMohriZhong2024deferral,MaoMohriZhong2024predictor,
  MaoMohriZhong2024score,mao2024h,mao2024multi,mao2024realizable,
  mao2024regression,mao2025enhanced,MaoMohriZhong2025mastering,
  MaoMohriZhong2025principled,MohriAndorChoiCollinsMaoZhong2024learning,cortes2024cardinality,cortes2025balancing,cortes2025improved,CortesMaoMohriZhong2026defid,CortesMohriZhong2026mod,mao2025theory,zhong2025fundamental,desalvo2025budgeted,montreuil2025two,montreuil2025adversarial,montreuil2026adversarial,montreuil2026why,montreuil2026optimal,montreuil2026beyond,montreuil2026online,montreuil2026learning,montreuil2026time,mohri2026principled,mohri2025beyond,MohriZhong2026slin,mohri2026generalized}. We
show that standard $\sH$-consistency fails for equicontinuous hypothesis sets
typical of neural networks. Instead, we propose the notion of
\emph{$\gamma$-shifted $\sH$-consistency} and prove that this guarantee holds
for the margin-shifted versions of DPO and IPO used in the LLM setting.

\paragraph{Margins and Structured Prediction.}
The concept of enforcing margins for consistency draws on the rich history of
Support Vector Machines (SVMs) \citep{cortes1995support} and their extension to
ranking \citep{herbrich2000large, joachims2002optimizing}. Our proposal for
\emph{Structure-Aware} margins is inspired by Structured SVMs
\citep{tsochantaridis2005large} and structured prediction
\citep*{mao2023structured}, as well as cost-sensitive classification
\citep{zadrozny2003cost}, which scale penalties based on the severity of the
error; we extend these ideas to the generative LLM setting. In the context of
LLMs, this connects to \emph{margin-aware} fine-tuning methods that have shown
empirical promise \citep{yuan2023rrhf, wang2023beyond} but lacked a unifying
theoretical consistency proof until now.

\paragraph{Social Choice and Game-Theoretic Perspectives.}
Recent theoretical works have also critically examined DPO from perspectives
beyond statistical consistency. \citet{golz2025distortion} use social choice
theory to show that DPO suffers from high distortion in aggregating diverse
human preferences, effectively acting as a Borda count. While our work focuses
on statistical consistency under a ground-truth assumption, our structure-aware
framework offers a potential mechanism to mitigate social distortion by
assigning lower margins to controversial pairs. Similarly, \citet{munos2023nash}
challenge the scalar reward assumption of the Bradley-Terry model, proposing
Nash Learning from Human Feedback (NLHF) to handle non-transitive preferences
via game-theoretic equilibria. Our work complements these approaches by focusing
on the optimization landscape: we address the inconsistency of the learning
objective itself, rather than the aggregation rule or the underlying preference
model.

\section{Preliminaries}
\label{sec:prelim}

\paragraph{Preference data.}
We consider the standard setting of LLM preference learning. Let $\sX$ be the
space of prompts and $\sY$ be the space of responses. We assume access to a
distribution $\sD$ over tuples $(x, y, y', w)$ with $y \neq y'$, where
$w \in \curl*{-1, 1}$ indicates the preference ($w = 1$ if $y \succ y'$). Let
\[
  \eta(x, y, y') = \mathbb{P} \paren*{w = 1 \mid x, y, y'}
\]
be the conditional preference probability.

\paragraph{Target loss and risks.}
The goal is to learn a scoring function $h \in \sH$ (e.g., a reward model or
implicit DPO log-ratio) minimizing the \emph{target 0-1 ranking loss}
\begin{equation}
  \label{eq:target-loss}
  \sfL_{0-1}(h, x, y, y', w) = 1_{w \neq \sign(\Delta h)},
  \qquad \text{where } \Delta h = h(x, y) - h(x, y'),
\end{equation}
and where we use the convention $\sign(0) = 1$. We denote the generalization
error by $\sR(h) = \E [\sfL_{0-1}]$ and the best-in-class error by
$\sR^*(\sH) = \inf_{h \in \sH} \sR(h)$. The \emph{conditional error} is
$\sC(h) = \E[\sfL_{0-1} \mid x, y, y']$, which induces the \emph{best-in-class
  conditional error} and the \emph{conditional regret}:
\[
  \sC^*(\sH) = \inf_{h \in \sH} \sC(h),
  \qquad
  \Delta \sC_{\sH}(h) = \sC(h) - \sC^*(\sH).
\]

\paragraph{Surrogate losses.}
Since optimizing $\sfL_{0-1}$ is intractable, practical algorithms minimize a
convex \emph{surrogate loss}
\[
  \sfL_{\Phi}(h, x, y, y', w) = \Phi(w \cdot \Delta h).
\]
For instance, DPO uses $\Phi_{\rm{log}}(u) = \log(1 + e^{-\beta u})$
\citep{chen2024preference,agarwaldesign}. We denote the surrogate generalization
error by $\sR_{\Phi}(h)$, with conditional terms $\sC_{\Phi}(h)$,
$\sC_{\Phi}^*(\sH)$ and $\Delta \sC_{\Phi, \sH}(h)$ defined analogously.

\paragraph{Minimizability gaps.}
We define the \emph{minimizability gaps} by
\begin{equation}
  \label{eq:min-gap}
  \sM_{\Phi}(\sH) = \sR_{\Phi}^*(\sH) - \E \bracket*{\sC_{\Phi}^*(\sH)},
  \qquad
  \sM(\sH) = \sR^*(\sH) - \E \bracket*{\sC^*(\sH)}.
\end{equation}
The gap measures the difference between the best-in-class error and the expected
best-in-class conditional error. It quantifies the approximation error
introduced by restricting our model to a specific hypothesis set $\sH$ (e.g.,
neural networks) rather than the space of all measurable functions
$\sH_{\rm{all}}$. As established in recent $\sH$-consistency literature
\citep{mao2024universal}, when $\sH = \sH_{\rm{all}}$, the minimizability gaps
vanish, and our bounds exactly recover the traditional excess risk bounds. A
smaller minimizability gap indicates that the restricted hypothesis class $\sH$
is well-aligned with the surrogate loss. Note that, by definition, for any
$h \in \sH$,
\begin{equation}
  \label{eq:regret-identity}
  \E\bracket*{\Delta \sC_{\sH}(h)} = \sR(h) - \sR^*(\sH) + \sM(\sH),
  \qquad
  \E\bracket*{\Delta \sC_{\Phi, \sH}(h)} = \sR_{\Phi}(h) - \sR_{\Phi}^*(\sH) + \sM_{\Phi}(\sH),
\end{equation}
an identity used repeatedly in our proofs.

\paragraph{Consistency.}
While classical \emph{Bayes-consistency}
\citep{Zhang2003,bartlett2006convexity,steinwart2007compare} ensures that
optimizing a surrogate over all measurable functions minimizes the target loss
asymptotically, it provides no guarantees for restricted hypothesis sets $\sH$
(e.g., neural networks). We therefore focus on \emph{$\sH$-consistency bounds}
\citep{awasthi2022h,MaoMohriZhong2023ranking,mao2023cross}, which provide
non-asymptotic guarantees specific to $\sH$. A surrogate $\Phi$ is
$\sH$-consistent if there exists a non-decreasing concave function $\Gamma$
(with $\Gamma(0)=0$) such that for all $h \in \sH$:
\begin{equation}
  \label{eq:H-consistency}
  \sR(h) - \sR^*(\sH) + \sM(\sH)
  \leq \Gamma \paren*{\sR_{\Phi}(h) - \sR_{\Phi}^*(\sH) + \sM_{\Phi}(\sH)}.
\end{equation}
This bound guarantees that minimizing the surrogate estimation error effectively
minimizes the target estimation error, accounting for the approximation
limitations captured by the minimizability gaps \citep{mao2024universal}.

\section{Inconsistency of Unconstrained Ranking}
\label{sec:inconsistency}

While convenient, surrogate risk minimization does not inherently guarantee
consistency with respect to the true target risk. We begin by formally defining
the necessary properties for the hypothesis set and then present a negative
result of $\sH$-consistency.

\subsection{Equicontinuity and Regularity}

We characterize the hypothesis set $\sH$ of scoring functions
$h\colon \sX \times \sY \to \Rset$ using two standard assumptions.

\begin{definition}[Equicontinuity]
  \label{def:equicontinuity}
  $\sH$ is \emph{equicontinuous} if, for any $\epsilon > 0$, there exist
  $x \in \sX$ and $y \neq y' \in \sY$ such that
  $\sup_{h \in \sH} |h(x, y) - h(x, y')| < \epsilon$.
\end{definition}

Equicontinuity implies that the hypothesis class can produce arbitrarily small
score differences for certain input pairs (e.g., when responses are
near-synonyms). Because modern LLMs represent continuous mappings in a
high-dimensional space, they naturally satisfy this property. Note that our
definition is a weaker, non-standard condition compared to uniform
equicontinuity. We chose this formulation precisely because it makes our
negative result (Theorem~\ref{Thm:negative-general}) mathematically stronger:
consistency completely fails even if the model only produces small score
differences on a tiny subset of inputs.

\begin{definition}[Regularity]
  \label{def:regularity}
  $\sH$ is \emph{regular} if, for any tuple $(x, y, y', w)$, there exists
  $h \in \sH$ such that $\sign(\Delta h) = w$.
\end{definition}

Regularity merely assumes the model class is expressive enough to correctly rank
a given pair (i.e., it can achieve the correct sign), which overparameterized
LLMs easily satisfy. Together, these assumptions highlight that without a
margin constraint, models can ``cheat'' by shrinking scores on ambiguous pairs
rather than learning robust rankings.

\subsection{A Vacuous \texorpdfstring{$\sH$}{H}-Consistency Bound}

We now demonstrate that for regular and equicontinuous hypothesis sets, any
$\sH$-consistency bound is inherently vacuous. Our result and proof technique
build upon the $\sH$-consistency framework and the fundamental inconsistency
findings for pairwise ranking introduced by \citet[Theorem
2.1]{MaoMohriZhong2023ranking}. We adapt and extend their analysis to
specifically address the equicontinuous hypothesis sets typical of deep
networks.

\begin{theorem}[Negative result for equicontinuous $\sH$]
  \label{Thm:negative-general}
  Assume that $\sH$ is \emph{regular}, \emph{equicontinuous}, and contains the
  zero function ($h_0 \equiv 0$). If a non-decreasing function
  $\Gamma\colon [0, \infty) \to [0, \infty)$, continuous at $0$, satisfies the
  $\sH$-consistency bound for all $h \in \sH$ and any distribution:
  \begin{equation*}
    \sR(h) - \sR^*(\sH) + \sM(\sH) \leq \Gamma \paren*{\sR_{\Phi}(h) - \sR_{\Phi}^*(\sH) + \sM_{\Phi}(\sH)},
  \end{equation*}
  then, $\Gamma(t) \geq 1$ for all $t \geq 0$.
\end{theorem}

The core obstruction is the decoupling of surrogate error and ranking error in
the absence of a margin. Consider a distribution concentrated on a single pair
with a deterministic preference. For this concentrated distribution, the
minimizability gaps vanish ($\sM(\sH) = \sM_{\Phi}(\sH) = 0$). By
equicontinuity, the support pair can be chosen so that every hypothesis squashes
the score difference $h(x,y) - h(x,y')$ arbitrarily close to 0. Since $\Phi$ is
continuous, the surrogate loss of the zero function then approaches the best
achievable value, while its 0-1 ranking error remains 1. Thus,
$\sR_\Phi \to \sR_\Phi^*$ does not imply $\sR \to \sR^*$: a margin $\gamma$ is
required to force the score difference away from this ambiguous region. We now
make this argument precise.

\begin{proof}
  Consider a distribution $\sD$ that is concentrated entirely on a single tuple
  $(x, y, y', w)$ with deterministic label $w = -1$ (implying $y' \succ y$). The
  true optimal error is
  $\sR^*(\sH) = \inf_{h \in \sH} 1_{w \neq \sign(h(x, y) - h(x, y'))}$. Since
  $\sH$ is regular, there exists $h \in \sH$ such that
  $-1 = \sign(h(x, y) - h(x, y'))$, implying $\sR^*(\sH) = 0$. Let $h_0$ be the
  zero function ($h_0(x, y) = 0$). Then $\sR(h_0) = 1_{-1 \neq 1} = 1$. By the
  \emph{equicontinuity} of $\sH$, for any $\e > 0$, we can choose the support
  elements $x, y \neq y'$ such that $\abs*{h(x, y) - h(x, y')} < \e$ for all
  $h \in \sH$.

  The surrogate error for any $h \in \sH$ is
  $\sR_{\Phi}(h) = \Phi(-1 \cdot (h(x, y) - h(x, y')))$. Then
  $\sR_{\Phi}(h_0) = \Phi(0)$. Since the function $\Phi$ is non-increasing, the
  range of possible surrogate errors is bounded:
  \[
    \sR_{\Phi}(h) \in \bracket*{\Phi(\e), \Phi(-\e)}.
  \]
  Therefore, the optimal surrogate error is lower-bounded by
  $\sR_{\Phi}^*(\sH) \geq \Phi(\e)$. Applying the assumed $\sH$-consistency
  bound to the zero function $h_0$, we get:
  \[
    \sR(h_0) - \sR^*(\sH) + \sM(\sH) \leq \Gamma \paren*{\sR_{\Phi}(h_0) -
      \sR_{\Phi}^*(\sH) + \sM_{\Phi}(\sH)},
  \]
  where $\sM(\sH) = \sM_{\Phi}(\sH) = 0$ for this distribution concentrated on a
  single tuple.  Substituting the known values and lower bound:
  \[
    1 - 0 \leq \Gamma \paren*{\Phi(0) - \sR_{\Phi}^*(\sH)} \leq \Gamma
    \paren*{\Phi(0) - \Phi(\e)}.
  \]
  Since $\Phi$ is continuous, taking the limit as $\e \to 0$ yields
  $\Phi(\e) \to \Phi(0)$. Given that $\Gamma$ is continuous at $0$ and
  non-decreasing, we conclude that $\Gamma(0) \geq 1$. By the non-decreasing
  nature of $\Gamma$, this implies $\Gamma(t) \geq 1$ for all $t \geq 0$.
\end{proof}

Theorem~\ref{Thm:negative-general} establishes that for hypothesis sets typical
of unconstrained neural networks, any $\sH$-consistency bound is
\emph{vacuous}: even when minimizability gaps vanish, a small surrogate
estimation error, $\sR_{\Phi}(h) - \sR_{\Phi}^*(\sH) \approx 0$, only
guarantees $\sR(h) - \sR^*(\sH) \leq 1$, providing no meaningful guarantee that
the true ranking error is small. This extends the inconsistency findings of
\citet[Theorem 2.1]{MaoMohriZhong2023ranking} to the equicontinuity properties
inherent to modern LLMs.

\section{\texorpdfstring{$\sH$}{H}-Consistency of Margin-Constrained Ranking}
\label{sec:margin-constrained}

To resolve the inconsistency, we restrict our analysis to hypothesis sets that
enforce a minimum separation margin $\gamma > 0$ between scores. Formally, we
define the family $\cH_\gamma$ of admissible hypothesis sets as the family of
sets $\sH \subseteq \sH_{\rm{all}}$ satisfying the following two conditions:
\begin{enumerate}[label=(\roman*)]
\item for all $h \in \sH$, $x \in \sX$ and $y \neq y'$:
  $\abs*{h(x, y) - h(x, y')} \geq \gamma$;
\item for all $x \in \sX$ and $y \neq y'$, there exist $h_+, h_- \in \sH$ such
  that $h_+(x, y) - h_+(x, y') = \gamma$ and $h_-(x, y) - h_-(x, y') = -\gamma$.
\end{enumerate}
This ensures every hypothesis separates pairs by at least $\gamma$, and the set
contains hypotheses $h_+$ and $h_-$ achieving this boundary exactly.

\subsection{A General \texorpdfstring{$\sH$}{H}-Consistency Bound}

Under this constraint, we can derive an $\sH$-consistency bound between the
surrogate loss and the target ranking loss for any hypothesis set
$\sH \in \cH_\gamma$.

\begin{theorem}[General $\sH$-consistency bound]
  \label{Thm:positive-general}
  Let $\sH \in \cH_\gamma$ be a regular hypothesis set. Let
  $\Phi\colon \Rset \to \Rset_{+}$ be a convex, non-increasing function. For any
  $h \in \sH$, the following $\sH$-consistency bound holds:
  \begin{equation*}
    \sR(h) - \sR^*\paren*{\sH} + \sM\paren*{\sH} \leq \frac{1}{\Phi(-\gamma) - \Phi(\gamma)} \paren*{\sR_{\Phi}(h) - \sR_{\Phi}^*\paren*{\sH} + \sM_{\Phi}\paren*{\sH}}.
  \end{equation*}
\end{theorem}

The proof relies on decomposing the conditional regret $\Delta\sC_{\sH}(h)$ into
surrogate terms. The critical step is establishing a pointwise lower bound on
the surrogate conditional error $\sC_{\Phi}(h)$ for any hypothesis that misranks
a pair. Due to the margin constraint on $\cH_\gamma$, any such misranking
implies a score violation of at least $\gamma$. By the convexity and
monotonicity of $\Phi$, we show that the surrogate regret scales with the
ranking regret:
$\Delta \sC_{\Phi, \sH}(h) \ge (\Phi(-\gamma) - \Phi(\gamma)) \Delta
\sC_{\sH}(h)$. Rearranging this inequality yields the consistency coefficient.

\begin{proof}
  Fix a tuple $(x, y, y')$ and let $\eta = \eta(x, y, y')$. The conditional
  error for the target loss is
  $\sC(h) = \eta 1_{\Delta h < 0} + (1 - \eta) 1_{\Delta h \geq 0}$. Since $\sH$
  is regular, the best-in-class conditional error is
  $\sC^*(\sH) = \min \curl*{\eta, 1 - \eta}$.
    
  Consider the set of hypotheses $\ov \sH(x, y, y')$ that misclassify the pair
  relative to the Bayes optimal decision. For any $h \notin \ov \sH(x, y, y')$,
  we have $\Delta \sC_{\sH}(h) = 0$ and the inequality we seek to establish
  holds trivially since $\Delta \sC_{\Phi, \sH}(h) \geq 0$. For any
  $h \in \ov \sH(x, y, y')$, the conditional regret is:
  \[
    \Delta \sC_{\sH}(h) = \sC(h) - \sC^*(\sH) = \abs*{2 \eta - 1}.
  \]
  For the surrogate loss, the conditional error is
  $\sC_{\Phi}(h) = \eta \Phi(\Delta h) + (1 - \eta) \Phi(-\Delta h)$. By the
  margin assumption, $\sH$ contains hypotheses $h_+, h_-$ achieving score
  differences $\Delta h = \pm \gamma$. The best-in-class surrogate conditional
  error is upper bounded by the minimum error among these:
  \begin{align*}
    \sC^*_{\Phi}(\sH) 
    & \leq \min\curl*{\sC_{\Phi}(h_{+}), \sC_{\Phi}(h_{-})}
      = \max \curl*{\eta, 1 - \eta} \Phi (\gamma)
      + \min \curl*{\eta, 1 - \eta} \Phi (-\gamma).
  \end{align*}
  Now, consider a misclassifying hypothesis $h \in \ov \sH(x, y, y')$. Without
  loss of generality, assume $\eta \geq 1/2$. For $h$ to misclassify, we must
  have $\Delta h < 0$. Furthermore, the $\gamma$-margin condition on $\sH$
  implies that the magnitude of this violation is at least $\gamma$, i.e.,
  $\Delta h \leq -\gamma$.
    
  Let $g(u) \coloneqq \eta \Phi(u) + (1 - \eta) \Phi(-u)$ be the surrogate error
  as a function of the score difference $u = \Delta h$. We show that $g(u)$ is
  non-increasing for $u < 0$. Consider scores $u_1 < u_2 \leq -\gamma < 0$. Let
  $S_1, S_2$ be the secant slopes of $\Phi$ on $[u_1, u_2]$ and $[-u_2, -u_1]$
  respectively. Since $\Phi$ is convex and non-increasing, we have
  $S_1 \leq S_2 \leq 0$. The change in error is:
  \begin{align*}
    g(u_1) - g(u_2) 
    &= \eta [\Phi(u_1) - \Phi(u_2)] + (1 - \eta) [\Phi(-u_1) - \Phi(-u_2)] \\
    &= (u_2 - u_1) \bracket*{ (1 - \eta)S_2 - \eta S_1 } \geq 0,
  \end{align*}
  where the inequality holds because $\eta \geq 1 - \eta$ and
  $-S_1 \geq -S_2 \geq 0$. Thus, $g(u)$ is minimized at the boundary
  $u = -\gamma$:
  \[
    \sC_{\Phi}(h) \geq g(-\gamma) = \eta \Phi(-\gamma) + (1 - \eta)
    \Phi(\gamma).
  \]
  Subtracting the upper bound for $\sC^*_{\Phi}(\sH)$ from this lower bound
  yields the surrogate conditional regret:
  \begin{align*}
    \Delta \sC_{\Phi, \sH}(h)
    &= \sC_{\Phi}(h) - \sC^*_{\Phi}(\sH) \\
    & \geq \bracket*{\eta \Phi(-\gamma) + (1 - \eta) \Phi(\gamma)}
      - \bracket*{\eta \Phi(\gamma) + (1 - \eta) \Phi(-\gamma)} \\
    & = (2 \eta - 1) \bracket*{\Phi(-\gamma) - \Phi(\gamma)} \\
    & = \paren*{\Phi(-\gamma) - \Phi(\gamma)} \Delta \sC_{\sH}(h).
  \end{align*} 
  Taking the expectation over $\sD$ and using \eqref{eq:regret-identity} yields
  the stated bound.
\end{proof}

Theorem~\ref{Thm:positive-general} establishes that in realizable settings where
minimizability gaps vanish, a surrogate estimation error of $\e$ guarantees that
the target estimation error $\sR(h) - \sR^*(\sH)$ is upper bounded by
$\frac{\e}{\Phi(-\gamma) - \Phi(\gamma)}$.

\paragraph{The case of DPO.}
The standard DPO formulation uses the logistic surrogate loss
$\Phi_{\rm{log}}(u) = \log(1 + e^{-\beta u})$ for some temperature parameter
$\beta > 0$.  We can obtain the specific bound for DPO by evaluating the
coefficient $\frac{1}{\Phi(-\gamma) - \Phi(\gamma)}$ explicitly. Substituting
the logistic function:
\[
  \Phi_{\rm{log}}(-\gamma) - \Phi_{\rm{log}}(\gamma)
  = \log \frac{1 + e^{\beta\gamma}}{1 + e^{-\beta\gamma}}
  = \log(e^{\beta \gamma}) = \beta \gamma.
\]
Thus, for the specific case of DPO, the general coefficient simplifies to
$\frac{1}{\beta \gamma}$.

\begin{corollary}[$\sH$-consistency bound for DPO]
  \label{Thm:positive-log}
  Let $\sH \in \cH_\gamma$ be a regular hypothesis set. For any $h \in \sH$ and
  the logistic surrogate
  $\Phi_{\rm{log}} \colon u \mapsto \log(1 + e^{-\beta u})$, the following
  $\sH$-consistency bound holds:
  \begin{equation*}
    \sR(h) - \sR^*\paren*{\sH} + \sM\paren*{\sH}
    \leq \frac{1}{\beta \gamma} \paren*{\sR_{\Phi_{\rm{log}}}(h) - \sR_{\Phi_{\rm{log}}}^*\paren*{\sH}
      + \sM_{\Phi_{\rm{log}}}\paren*{\sH}}.
  \end{equation*}
\end{corollary}

The coefficient $\frac{1}{\beta \gamma}$ indicates that a larger margin $\gamma$
tightens the bound, implying that surrogate minimization is more effective at
reducing the true ranking loss when the model maintains high confidence.
However, identifying a hypothesis set that strictly satisfies this margin in
practice is challenging. Similarly, a larger $\beta$ (lower temperature) reduces
the multiplicative factor, enhancing consistency, though excessively large
$\beta$ values may degrade the optimization landscape of $\Phi_{\rm{log}}$ via
vanishing gradients. Collectively, these results suggest that if an LLM (or
reward model) maintains a minimum confidence gap $\gamma$ between response
scores, minimizing the DPO objective serves as an effective proxy for minimizing
the true ranking error.

\subsection{Consistency on Finite Domains}
\label{sec:finite}

While Theorem~\ref{Thm:negative-general} establishes inconsistency for general
equicontinuous hypothesis sets, the specific domain of Large Language Models
offers a structural advantage: the input space consists of sequences over a
finite vocabulary. In practice, the support of the preference distribution $\sD$
is restricted to a finite set $S \subset \sX \times \sY \times \sY$.  We can
show that under the assumption of finite support and model realizability,
consistency is guaranteed because the hypothesis set of over-parameterized
neural networks can achieve an arbitrarily large separation margin.

\begin{definition}[Finite realizable support]
  The support $S = \supp(\sD)$ is finite. Furthermore, the distribution is
  realizable on $S$: for any $(x, y, y') \in S$ with $y \neq y'$, the preference
  label $w$ is deterministic ($|\eta(x, y, y') - 1/2| = 1/2$) and
  non-contradictory.
\end{definition}

\begin{definition}[Strict separability via logit scaling]
  We assume the hypothesis set $\sH$ is parameterized by logits
  $z_\theta(x, \cdot)$ such that the policy is
  $\pi_\theta(y|x) \propto \exp(z_\theta(x, y))$. We say $\sH$ is \emph{strictly
    separable} on $S$ if there exists a parameter $\theta$ such that for all
  $(x, y, y', w) \in S$:
  \[
    z_\theta(x, y_{\text{win}}) > z_\theta(x, y_{\text{lose}}),
  \]
  where $y_{\text{win}}$ is the preferred response. Furthermore, we assume $\sH$
  is closed under positive scalar multiplication of the logits ($z \to \alpha z$
  for $\alpha > 0$).
\end{definition}

\begin{theorem}[Consistency via margin scaling]
  \label{Thm:finite-consistency}
  Under the assumptions of finite realizable support and strict separability,
  minimizing the DPO logistic surrogate loss is consistent with respect to the
  0-1 ranking loss. Specifically, for any error tolerance $\epsilon > 0$, there
  exists a hypothesis $h \in \sH$ (achieved by scaling logits) such that the
  $\sH$-consistency bound holds with a coefficient sufficiently small to
  guarantee:
  \[
    \sR_{\Phi_{\rm{log}}}(h) \to 0 \implies \sR(h) = 0.
  \]
\end{theorem}

\begin{proof}
  Since $\sH$ is strictly separable on the finite set $S$, there exists a base
  parameter $\theta_0$ and a minimum raw margin $\delta > 0$ such that the logit
  difference
  $z_{\theta_0}(x, y_{\text{win}}) - z_{\theta_0}(x, y_{\text{lose}}) \ge
  \delta$ for all tuples in $S$.

  Consider the scaled logits $z_\alpha = \alpha z_{\theta_0}$ for $\alpha >
  0$. As $\alpha \to \infty$, the policy $\pi_\alpha$ converges to a hard argmax
  distribution:
  \[
    \lim_{\alpha \to \infty} \pi_\alpha(y_{\text{win}} \mid x) = 1, \quad
    \lim_{\alpha \to \infty} \pi_\alpha(y_{\text{lose}} \mid x) = 0.
  \]
  The DPO implicit reward difference is given by:
  \[
    \Delta h(x, y, y') = \beta \log \frac{\pi_\alpha(y_{\text{win}} \mid
      x)}{\pi_{\text{ref}}(y_{\text{win}} \mid x)} - \beta \log
    \frac{\pi_\alpha(y_{\text{lose}} \mid x)}{\pi_{\text{ref}}(y_{\text{lose}}
      \mid x)}.
  \]
  The term $-\log \pi_\alpha(y_{\text{lose}} \mid x)$ dominates the expression,
  driving the magnitude of the score difference to infinity:
  \[
    \lim_{\alpha \to \infty} w \cdot (h_\alpha(x, y) - h_\alpha(x, y')) =
    +\infty.
  \]
  Consequently, for any finite dataset $S$, we can choose a scaling factor
  $\alpha$ large enough such that the effective margin $\gamma$ on all examples
  exceeds any arbitrary threshold. Applying the bound from
  Corollary~\ref{Thm:positive-log}, the consistency coefficient
  $\frac{1}{\beta \gamma}$ vanishes as $\gamma \to \infty$. Thus, on finite
  realizable domains, the ``vacuous'' bound identified in the negative result is
  overcome by the capacity of the model to drive margins to infinity.
\end{proof}

\subsection{Topological Constraints and Non-Emptiness}
\label{sec:topological}

The class $\cH_\gamma$ is technically non-empty for discrete domains or
hypothesis sets with discontinuities, such as decision trees. Concrete
instantiations include:
\begin{itemize}
\item \emph{Global Preference Models}, where functions $h(x, y) = s_y$ have
  fixed scores separated by $\gamma$;
\item \emph{Discontinuous Models} like decision trees that partition $\sX$ into
  regions, allowing preferences to flip discontinuously at boundaries while
  satisfying local margins;
\item \emph{Quantized Models}, such as neural networks with discrete output
  activations (e.g., $h(x, y) \in \curl*{k \gamma \colon k \in \Zset}$) that
  ensure all distinct score differences are multiples of $\gamma$.
\end{itemize}

However, for continuous hypothesis sets such as neural networks on connected
input domains $\sX$, the strict margin condition
$|h(x, y) - h(x, y')| \geq \gamma$ imposes a severe topological constraint. By
the Intermediate Value Theorem, if the model were to change its preference
between $y$ and $y'$ as the input $x$ varies, the score difference would
necessarily pass through zero, violating the $\gamma$-condition. Consequently,
Theorem~\ref{Thm:positive-general} strictly applies only to models with fixed
global preferences or discontinuous decision boundaries. While this establishes
the theoretical \emph{sufficiency} of margins for consistency, enforcing
$\cH_\gamma$ directly is too restrictive for practical deep learning. This
limitation motivates our proposal of \emph{Margin-Shifted Surrogates} in
Section~\ref{sec:shifted}, which enforce margins via a soft penalty in the loss
rather than a hard constraint on $\sH$. This transition from a ``hard'' margin
to a ``soft'' margin-shifted surrogate introduces a margin approximation gap
$\sA_\gamma(\sH)$ in the bounds that formally quantifies the trade-off: we gain
topological flexibility but incur an approximation penalty proportional to the
model's inability to fully satisfy the margin.

\section{\texorpdfstring{$\gamma$}{gamma}-Approximate
  \texorpdfstring{$\sH$}{H}-Consistency via Margin-Shifted Surrogates}
\label{sec:shifted}

Strict margin constraints ensure consistency but often render the hypothesis set
non-convex. To maintain computational tractability, we propose the
\emph{margin-shifted surrogate loss}. Instead of constraining $\sH$ directly, we
shift the loss function to penalize correct classifications that fail to achieve
a target confidence margin $\gamma > 0$.

\begin{definition}[Margin-shifted surrogate]
  Let $\Phi\colon \Rset \to \Rset_+$ be a convex, non-increasing surrogate
  function. For a target margin $\gamma > 0$, we define the margin-shifted
  surrogate loss $\sfL_{\Phi_\gamma}$ as:
  \begin{equation}
    \label{eq:shifted-loss}
    \sfL_{\Phi_\gamma}(h, x, y, y', w)
    = \Phi\paren*{w \cdot (h(x, y) - h(x, y')) - \gamma}.
  \end{equation}
\end{definition}

\subsection{\texorpdfstring{$\gamma$}{Gamma}-Approximate
  \texorpdfstring{$\sH$}{H}-Consistency Bound}

Crucially, since the argument $u \mapsto u - \gamma$ is affine and $\Phi$ is
convex, the composition $\sfL_{\Phi_\gamma}$ remains convex with respect to
$h$. This allows for standard gradient-based optimization without the
feasibility issues of non-convex constraints.  We now derive a
$\gamma$-approximate $\sH$-consistency bound for this loss.

\begin{theorem}[$\gamma$-shifted $\sH$-consistency bound]
  \label{th:gamma-shifted-H-consistency}
  Let $\Phi \colon \Rset \to [0, +\infty)$ be non-increasing with
  $\Phi(-\gamma) > 0$ for some $\gamma > 0$. Define the shifted surrogate
  $\Phi_\gamma(u) = \Phi(u - \gamma)$. Then, for all $h \in \sH$:
  \begin{equation*}
    \sR(h) - \sR^*(\sH) + \sM(\sH)
    \leq \frac{1}{\Phi(-\gamma)}
    \bracket*{\sR_{\Phi_\gamma}(h) - \sR_{\Phi_\gamma}^*(\sH)
      + \sM_{\Phi_\gamma}(\sH)} + \sA_{\gamma}(\sH),
  \end{equation*}
  where
  $\sA_{\gamma}(\sH) = \frac{\E[\sC^*_{\sfL_{\Phi_\gamma}}(\sH)]}{\Phi(-\gamma)}
  - \E[\sC^*(\sH)]$ is the \emph{margin approximation gap}.
\end{theorem}

We use a calibration technique relating the 0-1 loss to the convex
surrogate. The key insight is to upper-bound the discontinuous 0-1 indicator
$1_{u \le 0}$ using the shifted surrogate $\Phi(u - \gamma)$: by monotonicity
of $\Phi$, if a pair is misranked ($u \le 0$), then $u - \gamma \le -\gamma$,
implying $\Phi(u-\gamma) \ge \Phi(-\gamma)$. Taking expectations over the
preference label $w$ then yields a bound on the \emph{conditional error}
$\sC(h)$ in terms of the shifted surrogate conditional error
$\sC_{\Phi_\gamma}(h)$, and we decompose the conditional regret
$\Delta \sC_{\sH}$ into a scaled surrogate estimation term and the
approximation gap $\sA_\gamma$. This gap arises because the expected
best-in-class surrogate conditional error $\E[\sC^*_{\Phi_\gamma}(\sH)]$ (which
enforces margins) can be strictly higher than the unconstrained target error
$\E[\sC^*(\sH)]$ (which only requires correct signs).

\begin{proof}
  Fix $h \in \sH$. For any tuple $(x, y, y', w)$, let
  $\Delta h(x) \coloneqq w (h(x, y) - h(x, y'))$. The 0-1 pairwise ranking loss
  is upper bounded by $1_{\Delta h(x) \leq 0}$. Using the definition
  $\Phi_\gamma(u) = \Phi(u-\gamma)$ and the monotonicity of $\Phi$:
  \[
    u \leq 0 \implies u - \gamma \leq -\gamma \implies \Phi_\gamma(u) = \Phi(u -
    \gamma) \geq \Phi(-\gamma) > 0.
  \]
  Thus, we have the pointwise inequality
  $1_{u\leq 0} \leq \frac{\Phi_\gamma(u)}{\Phi(-\gamma)}$. Applying this with
  $u = \Delta h(x)$:
  \begin{equation*}
    \sfL_{0-1}(h, x, y, y', w) \leq \frac{\sfL_{\Phi_\gamma}(h, x, y, y', w)}{\Phi(-\gamma)}.
  \end{equation*}
  Taking expectations over $w$ yields
  $\sC(h) \leq \frac{\sC_{\sfL_{\Phi_\gamma}}(h)}{\Phi(-\gamma)}$ (where we
  suppress the arguments $x, y, y'$ for brevity). Since this holds for all
  $h \in \sH$, it holds for the infimum:
  \[
    \sC^*(\sH) \leq \inf_{h \in \sH}
    \frac{\sC_{\sfL_{\Phi_\gamma}}(h)}{\Phi(-\gamma)} =
    \frac{\sC^*_{\sfL_{\Phi_\gamma}}(\sH)}{\Phi(-\gamma)}.
  \]
  This implies the approximation gap is non-negative:
  $\sA_{\gamma}(\sH) \geq 0$. We now decompose the conditional regret:
  \begin{align*}
    \Delta \sC_{\sH}(h)
    &= \sC(h) - \sC^*(\sH) \\
    &\leq \frac{\sC_{\sfL_{\Phi_\gamma}}(h)}{\Phi(-\gamma)} - \sC^*(\sH) \\
    &= \frac{1}{\Phi(-\gamma)}\bracket*{\sC_{\sfL_{\Phi_\gamma}}(h) - \sC^*_{\sfL_{\Phi_\gamma}}(\sH)} + \bracket*{\frac{\sC^*_{\sfL_{\Phi_\gamma}}(\sH)}{\Phi(-\gamma)} - \sC^*(\sH)} \\
    &= \frac{1}{\Phi(-\gamma)} \Delta \sC_{\sfL_{\Phi_\gamma}, \sH}(h) + \bracket*{\frac{\sC^*_{\sfL_{\Phi_\gamma}}(\sH)}{\Phi(-\gamma)} - \sC^*(\sH)}.
  \end{align*}
  Taking the expectation over $(x, y, y')$ and using
  \eqref{eq:regret-identity} yields the final result:
  \[
    \sR(h) - \sR^*(\sH) + \sM(\sH) \leq
    \frac{1}{\Phi(-\gamma)}\bracket*{\sR_{\Phi_\gamma}(h) -
      \sR_{\Phi_\gamma}^*(\sH) + \sM_{\Phi_\gamma}(\sH)} + \sA_{\gamma}(\sH).
    \qedhere
  \]
\end{proof}

The approximation gap $\sA_{\gamma}(\sH)$ represents the price paid for
requiring the model to satisfy the margin $\gamma$. Even if the expected
best-in-class target conditional error $\E\bracket*{\sC^*(\sH)}$ is zero,
$\sA_{\gamma}(\sH)$ may be positive if the hypothesis set cannot produce scores
with magnitude at least $\gamma$. However, as shown in
Theorem~\ref{Thm:finite-consistency}, for strictly separable finite data and
scale-invariant models (like neural networks with unbounded logits), we can
drive $\sA_{\gamma}(\sH) \to 0$ by scaling the outputs (see also
Proposition~\ref{prop:vanishing-gap}).

\subsection{Properties of Shifted Consistency}

\begin{proposition}[Vacuousness under multiplicative shift invariance]
  \label{prop:vacuous-shift}
  If a surrogate $\Phi$ satisfies $\Phi(u - \gamma) = C(\gamma) \Phi(u)$ for
  some $C(\gamma) > 0$, then for any $\gamma > 0$, the $\gamma$-shifted bound
  collapses to the unshifted bound ($\gamma=0$). Specifically, the margin
  parameter cancels out, yielding:
  \[
    \sR(h) - \sR^*(\sH) + \sM(\sH) \leq \frac{1}{\Phi(0)} \paren*{ \sR_{\Phi}(h)
      - \sR_{\Phi}^*(\sH) + \sM_{\Phi}(\sH)} + \sA_{0}(\sH).
  \]
\end{proposition}

\begin{proof}
  Assume the invariance property holds: $\Phi(u - \gamma) = C(\gamma)\Phi(u)$.
  First, evaluating at $u=0$ implies $\Phi(-\gamma) = C(\gamma)\Phi(0)$. Thus,
  the consistency coefficient becomes:
  \[
    \frac{1}{\Phi(-\gamma)} = \frac{1}{C(\gamma)\Phi(0)}.
  \]
  Second, the shifted surrogate error scales linearly:
  \[
    \sR_{\Phi_\gamma}(h) = \E[\Phi(w \Delta h - \gamma)] = \E[C(\gamma)\Phi(w
    \Delta h)] = C(\gamma)\sR_{\Phi}(h).
  \]
  This scaling applies strictly to the optimal error
  $\sR^*_{\Phi_\gamma}(\sH) = C(\gamma)\sR^*_{\Phi}(\sH)$ and the minimizability
  gap $\sM_{\Phi_\gamma}(\sH) = C(\gamma)\sM_{\Phi}(\sH)$. Consequently, the
  bracketed estimation term becomes:
  \[
    \sR_{\Phi_\gamma}(h) - \sR_{\Phi_\gamma}^*(\sH) + \sM_{\Phi_\gamma}(\sH) =
    C(\gamma)\bracket*{\sR_{\Phi}(h) - \sR_{\Phi}^*(\sH) + \sM_{\Phi}(\sH)}.
  \]
  Third, the approximation gap $\sA_{\gamma}(\sH)$ simplifies:
  \[
    \sA_{\gamma}(\sH) =
    \frac{\E[\sC^*_{\sfL_{\Phi_\gamma}}(\sH)]}{\Phi(-\gamma)} - \E[\sC^*(\sH)] =
    \frac{C(\gamma)\E[\sC^*_{\sfL_{\Phi}}(\sH)]}{C(\gamma)\Phi(0)} -
    \E[\sC^*(\sH)] = \sA_{0}(\sH).
  \]
  Substituting these components back into the $\gamma$-shifted bound
  (Theorem~\ref{th:gamma-shifted-H-consistency}), the factor $C(\gamma)$ cancels
  in the first term, yielding the stated unshifted bound.
\end{proof}

This highlights a critical failure mode: for the exponential loss
($\Phi(u) = e^{-u}$), shifting the margin scales the loss by $e^\gamma$, which
cancels in the normalized bound. This explains why margin shifts are ineffective
for purely exponential losses compared to logistic or hinge losses.

\begin{proposition}[Tightness of the $\gamma$-shifted $\sH$-consistency
  constant]
  \label{prop:tightness}
  Let $\Phi$ be non-increasing with $\lim_{u \to \infty}\Phi(u) = 0$. The
  coefficient $C = \frac{1}{\Phi(-\gamma)}$ is optimal: for any $C' < C$, there
  exist a distribution $\sD$ and hypothesis set $\sH$ such that the consistency
  bound fails for some $h \in \sH$:
  \[
    \sR(h) - \sR^*(\sH) + \sM(\sH) > C' \bracket*{\sR_{\Phi_\gamma}(h) -
      \sR_{\Phi_\gamma}^*(\sH) + \sM_{\Phi_\gamma}(\sH)} + \sA_{\gamma}(\sH).
  \]
\end{proposition}

\begin{proof}
  Consider a distribution $\sD$ supported on a single tuple $(x, y, y', w)$ with
  $w = -1$. The Bayes ranking error is $0$. For any hypothesis $h$, write
  $u_h = w \paren*{h(x, y) - h(x, y')}$ for the signed score difference, so that
  the 0-1 loss of $h$ is $1_{w \neq \sign(h(x,y) - h(x,y'))}$ and its shifted
  surrogate loss is $\Phi(u_h - \gamma)$.

  Let $\sH$ contain a hypothesis $h_{\text{boundary}}$ with score difference
  $h_{\text{boundary}}(x, y) - h_{\text{boundary}}(x, y') = 0$. Its signed
  score difference is $0$ and its shifted score difference is $-\gamma$. Then:
  \[
    \sR(h_{\text{boundary}}) = 1_{-1 \neq \sign(0)} = 1, \quad
    \sR_{\Phi_\gamma}(h_{\text{boundary}}) = \Phi(0 - \gamma) = \Phi(-\gamma).
  \]
  Let the class $\sH$ also contain a sequence of functions $h_t$ with score
  differences $h_t(x, y) - h_t(x, y') = -t$, that is, signed score differences
  $u_{h_t} = t$, with $t \to +\infty$. Then:
  \[
    \sR(h_t) = 0, \quad \sR_{\Phi_\gamma}(h_t) = \Phi(t - \gamma).
  \]
  As $t \to +\infty$, we have $\sR_{\Phi_\gamma}(h_t) \to 0$. Thus,
  $\sR^*(\sH) = 0$ and $\sR_{\Phi_\gamma}^*(\sH) = 0$.  Since the distribution
  is supported on a single tuple, we have
  $\E\bracket*{\sC^*_{\sfL_{\Phi_\gamma}}(\sH)} = \sR_{\Phi_\gamma}^*(\sH) = 0$
  and $\E\bracket*{\sC^*(\sH)} = \sR^*(\sH) = 0$, so the approximation gap
  vanishes:
  \[
    \sA_{\gamma}(\sH) = \frac{0}{\Phi(-\gamma)} - 0 = 0.
  \]
  Also, the minimizability gaps vanish: $\sM(\sH) = \sM_{\Phi_\gamma}(\sH) = 0$.
  Substituting these values into the inequality with a hypothetical constant
  $C'$:
  \[
    \sR(h_{\text{boundary}}) \leq C' \sR_{\Phi_\gamma}(h_{\text{boundary}}) + 0
    \Leftrightarrow 1 \leq C' \Phi(-\gamma) \Leftrightarrow C' \geq
    \frac{1}{\Phi(-\gamma)}.
  \]
  Thus, no constant smaller than $\frac{1}{\Phi(-\gamma)}$ can satisfy the bound
  universally.
\end{proof}

This confirms that the trade-off between margin size and consistency is
fundamental: enforcing a margin $\gamma$ necessarily incurs a factor of
$1/\Phi(-\gamma)$, motivating surrogates like the Cubic Hinge
(Section~\ref{sec:poly-hinge}) that decay rapidly to minimize this cost.

\subsection{Structure-Aware \texorpdfstring{$\sH$}{H}-Consistency Bounds}
\label{sec:structure-aware}

The previous analysis assumed a uniform margin parameter $\gamma$ across all
inputs.  However, in ranking tasks involving structured objects like text
sequences, the difficulty of distinguishing a pair $(y, y')$ varies
significantly based on their semantic similarity.  Demanding a large margin for
nearly identical responses introduces unnecessary bias, while a small global
margin loosens the consistency guarantee for distinct pairs.  To address this,
we introduce a \emph{pair-dependent margin function}
$\Gamma \colon \sY \times \sY \to \Rset_+$.  This allows us to enforce margins
proportional to the dissimilarity of the responses, a technique akin to margin
scaling in Structured SVMs \citep{tsochantaridis2005large} and structured
prediction in general \citep{mao2023structured}.

\begin{definition}[Structure-aware margin-shifted surrogate]
  Let $\Delta(y, y')$ be a non-negative distance metric between responses (e.g.,
  normalized edit-distance or semantic embedding distance).  We define the
  structure-aware margin as $\Gamma(y, y') = \tau \Delta(y, y')$ for a scaling
  factor $\tau > 0$.  The corresponding surrogate loss is:
  \begin{equation}
    \label{eq:structure-loss}
    \sfL_{\Phi, \Gamma}(h, x, y, y', w) = \Phi\paren*{ w \cdot (h(x, y) - h(x,
      y')) - \Gamma(y, y') }.
  \end{equation}
\end{definition}

\paragraph{Algorithm: Structure-Aware DPO (SA-DPO).}
By applying this framework to the standard DPO logistic loss, we obtain a novel
objective we term SA-DPO. Substituting $\Phi_{\rm log}$ and the implicit reward
formulation, the objective becomes:
\begin{align}
  \cL_{\text{SA-DPO}}(\pi_\theta)
  & = -\E \bracket*{ \log \sigma \paren*{ \beta w \cdot \Delta h_\theta(x) - \tau \Delta(y, y') } } \nonumber \\
  & = -\E \bracket*{ \log \sigma \paren*{ \beta w \log \frac{\pi_\theta(y|x)\pi_{\text{ref}}(y'|x)}{\pi_\theta(y'|x)\pi_{\text{ref}}(y|x)} - \tau \Delta(y, y') } }.
  \label{eq:sa-dpo}
\end{align}
Unlike standard DPO, which pushes for a constant log-probability gap regardless
of content, SA-DPO dynamically relaxes the margin constraint for a semantically
similar pair $(y, y')$.

\paragraph{Concrete instantiations of $\Gamma$.}
The choice of $\Delta(y, y')$ allows experts to inject prior knowledge into the
alignment process:
\begin{itemize}
\item \emph{Semantic Embedding Distance}, defined as
  $\Delta(y, y') = 1 - \cos(E(y), E(y'))$, ensures the model learns strong
  discrimination only when responses are semantically distinct, preventing
  ``margin collapse'' on synonyms;
\item \emph{Edit Distance} is suitable for code generation or exact formatting
  tasks, where normalized edit distance ensures that small syntax errors are
  penalized less aggressively than complete hallucinations;
\item \emph{Gold Reward Gap}: if a teacher reward model $R^*$ is available,
  setting $\Gamma \propto |R^*(x,y) - R^*(x,y')|$ recovers a ranking-consistent
  formulation of Knowledge Distillation, forcing the student to respect the
  teacher's confidence gap rather than just the teacher's label.
\end{itemize}

The following result provides a guarantee for structure-aware algorithms.

\begin{theorem}[Structure-aware $\gamma$-shifted $\sH$-consistency bound]
  \label{th:structure-gamma-shifted-H-consistency}
  Let $\Phi \colon \Rset \to [0, \infty)$ be non-increasing with $\Phi(0) > 0$,
  and let $\Gamma(y, y') > 0$. Define the \emph{inverse-margin weighted loss}
  $\wt \sfL_{\Phi, \Gamma} = \frac{\sfL_{\Phi, \Gamma}}{\Phi(-\Gamma)}$. Then,
  for any $h \in \sH$:
  \begin{equation*}
    \sR(h) - \sR^*(\sH) + \sM(\sH)
    \leq \bracket*{ \sR_{\wt \sfL_{\Phi, \Gamma}}(h) - \sR^*_{\wt \sfL_{\Phi, \Gamma}}(\sH)
      + \sM_{\wt \sfL_{\Phi, \Gamma}}(\sH) } + \sA_{\Gamma}(\sH),
  \end{equation*}
  where
  $\sA_{\Gamma}(\sH) = \E[\sC^*_{\wt \sfL_{\Phi, \Gamma}}(\sH)] -
  \E[\sC^*(\sH)]$ is the structure-aware approximation gap.
\end{theorem}

We extend the calibration technique from
Theorem~\ref{th:gamma-shifted-H-consistency} by introducing a pair-dependent
local margin $\gamma_{\rm{loc}} = \Gamma(y, y')$. The key insight is that the
inverse-margin weighted loss $\wt \sfL_{\Phi, \Gamma}$ acts as a normalized
upper bound on the 0-1 loss:
$1_{\Delta h \le 0} \le \frac{\Phi(\Delta h -
  \gamma_{\rm{loc}})}{\Phi(-\gamma_{\rm{loc}})} = \wt \sfL_{\Phi,
  \Gamma}$. Taking expectations allows us to bound the target error $\sR(h)$ by
the weighted surrogate error $\sR_{\wt \sfL}(h)$, where the approximation gap
$\sA_\Gamma$ now captures the expected difficulty of satisfying these variable
margins across the distribution.

\begin{proof}
  Fix $h \in \sH$. We begin with the pointwise analysis. Let
  $u = w \cdot (h(x, y) - h(x, y'))$ be the signed score difference. The 0-1
  pairwise ranking loss is upper bounded by $1_{u \leq 0}$. For a specific pair
  $(y, y')$, let $\gamma_{\rm{loc}} = \Gamma(y, y')$. Since $\Phi$ is
  non-increasing and $\gamma_{\rm{loc}} > 0$:
  \[
    u \leq 0 \implies u - \gamma_{\rm{loc}} \leq -\gamma_{\rm{loc}} \implies
    \Phi(u - \gamma_{\rm{loc}}) \geq \Phi(-\gamma_{\rm{loc}}) \geq \Phi(0) > 0.
  \]
  Dividing by the positive term $\Phi(-\gamma_{\rm{loc}})$ and using the
  definition of the inverse-margin weighted loss $\wt \sfL_{\Phi, \Gamma}$, we
  obtain the pointwise bound:
  \[
    \sfL_{0-1}(h, x, y, y', w) \leq 1_{u \leq 0} \leq \frac{\Phi(u -
      \gamma_{\rm{loc}})}{\Phi(-\gamma_{\rm{loc}})} = \wt \sfL_{\Phi, \Gamma}(h,
    x, y, y', w).
  \]
  Taking expectations over $w$ yields the inequality between conditional errors
  $\sC(h) \leq \sC_{\wt \sfL_{\Phi, \Gamma}}(h)$.  Since this holds for any
  $h \in \sH$, it also holds for the infimum:
  \[
    \sC^*(\sH) = \inf_{h \in \sH} \sC(h) \leq \inf_{h \in \sH} \sC_{\wt
      \sfL_{\Phi, \Gamma}}(h) = \sC_{\wt \sfL_{\Phi, \Gamma}}^*(\sH).
  \]
  This implies the approximation gap is non-negative:
  $\sA_{\Gamma}(\sH) \geq 0$. We now decompose the target conditional
  regret. For any $x, y, y'$:
  \begin{align*}
    \Delta \sC_{\sH}(h) 
    &= \sC(h) - \sC^*(\sH) \\
    &\leq \sC_{\wt \sfL_{\Phi, \Gamma}}(h) - \sC^*(\sH) \\
    &= \bracket*{ \sC_{\wt \sfL_{\Phi, \Gamma}}(h) - \sC_{\wt \sfL_{\Phi, \Gamma}}^*(\sH) } + \bracket*{ \sC_{\wt \sfL_{\Phi, \Gamma}}^*(\sH) - \sC^*(\sH) } \\
    &= \Delta \sC_{\wt \sfL_{\Phi, \Gamma}, \sH}(h) + \bracket*{ \sC_{\wt \sfL_{\Phi, \Gamma}}^*(\sH) - \sC^*(\sH) }.
  \end{align*}
  Taking the expectation over $(x, y, y')$:
  \begin{itemize}
  \item the left-hand side becomes
    $\E[\Delta \sC_{\sH}(h)] = \sR(h) - \sR^*(\sH) + \sM(\sH)$;
  \item the first term on the right-hand side becomes
    $\E[\Delta \sC_{\wt \sfL_{\Phi, \Gamma}, \sH}(h)] = \sR_{\wt \sfL_{\Phi,
        \Gamma}}(h) - \sR^*_{\wt \sfL_{\Phi, \Gamma}}(\sH) + \sM_{\wt
      \sfL_{\Phi, \Gamma}}(\sH)$;
  \item the second term on the right-hand side becomes the structure-aware
    approximation gap $\sA_{\Gamma}(\sH)$.
  \end{itemize}
  Combining these yields the stated bound.
\end{proof}

\paragraph{Benefit of structure-awareness.}
The uniform bound (Theorem~\ref{th:gamma-shifted-H-consistency}) is dominated by
the worst-case margin pair, scaling with $1/\Phi(-\min_{y, y'} \Gamma(y, y'))$.
If the distribution contains even one pair of very similar responses
($\min_{y, y'} \Gamma(y, y') \to 0$), the uniform coefficient explodes,
rendering the bound loose.  In contrast, the structure-aware bound depends on
the \emph{expectation} of the inverse margins.  As long as the average pair is
distinct, the consistency guarantee remains tight, even if hard pairs exist in
the tail of the distribution. This justifies using weighted losses (like margin
scaling) to focus optimization on pairs where the semantic difference warrants a
strong ranking signal. Figure~\ref{fig:structure_aware} illustrates the benefit
of structure-aware margins.

\section{Analysis of the Margin-Capacity Profile}
\label{sec:margin-capacity}

The $\gamma$-shifted $\sH$-consistency bound introduced in
\cref{th:gamma-shifted-H-consistency} presents a fundamental trade-off:
increasing the margin $\gamma$ improves the consistency coefficient
$\frac{1}{\Phi(-\gamma)}$, but potentially increases the surrogate risk if the
hypothesis set $\sH$ lacks the capacity to fully satisfy the margin.

To analyze this trade-off precisely, we revisit the pointwise inequality
established in the proof of \cref{th:gamma-shifted-H-consistency}.  For any
hypothesis $h$, the true ranking error is directly upper-bounded by the
normalized shifted surrogate error:
\begin{equation}
  \label{eq:simple_bound}
  \sR(h) \leq \frac{\sR_{\Phi_\gamma}(h)}{\Phi(-\gamma)}.
\end{equation}
This simplified bound highlights that the guarantee is governed by the ratio of
the achieved loss to the margin penalty.  In the practical regime where models
have bounded outputs (e.g., logits constrained by normalization or fixed
architectures), the numerator $\sR_{\Phi_\gamma}(h)$ cannot be driven to zero if
$\gamma$ exceeds the model's maximum output scale.

\subsection{Analysis of the Margin Approximation Gap
  \texorpdfstring{$\sA_{\gamma}(\sH)$}{A\_gamma}}
\label{subsec:gap_analysis}

The approximation gap term $\sA_{\gamma}(\sH)$ defined in
Theorem~\ref{th:gamma-shifted-H-consistency} acts as a \emph{margin penalty}. It
measures the discrepancy between the expected best possible conditional error
under the strict margin requirement (surrogate conditional error) and the
expected best possible conditional ranking error (target conditional error).  We
analyze its behavior in two key regimes: unbounded capacity (where the gap
vanishes) and bounded capacity (where the gap creates an irreducible error
floor).

\begin{proposition}[Vanishing gap under infinite capacity]
  \label{prop:vanishing-gap}
  Let $\Phi \colon \Rset \to [0, \infty)$ satisfy $\Phi(-\gamma) > 0$ and
  $\lim_{u \to + \infty}\Phi(u) = 0$.  Let $\sH$ be a hypothesis set closed
  under positive scalar multiplication (i.e.,
  $h \in \sH \implies \alpha h \in \sH$ for all $\alpha > 0$).  If $\sH$ is
  capable of perfect ranking on the support of $\sD$ (i.e., $\sR^*(\sH) = 0$),
  then for any finite margin $\gamma > 0$:
  \begin{equation*}
    \lim_{\alpha \to \infty} \sA_{\gamma}(\sH_\alpha) = 0,
  \end{equation*}
  where $\sH_\alpha$ denotes the scaled hypothesis set.
\end{proposition}

\begin{proof}
  Since $\sR^*(\sH) = 0$, there exists $h \in \sH$ such that
  $w \cdot (h(x, y) - h(x, y')) > 0$ almost surely.  Let
  $u(x) = w \cdot (h(x, y) - h(x, y'))$ denote this signed score difference.
  Consider the scaled hypothesis $h_\alpha = \alpha h$. The shifted surrogate
  error is:
  \[
    \sR_{\Phi_\gamma}(h_\alpha) = \E_{(x, y, y', w) \sim \sD}
    \bracket*{\Phi(\alpha u(x) - \gamma)}.
  \]
  Since $u(x) > 0$, as $\alpha \to +\infty$, the argument
  $\alpha u(x) - \gamma \to +\infty$.  Since $\lim_{u \to +\infty} \Phi(u) = 0$,
  by the dominated convergence theorem, $\sR_{\Phi_\gamma}(h_\alpha) \to 0$.
  Consequently, we have $\sR_{\Phi_\gamma}^*(\sH) = \sR^*(\sH) =0$.  This
  implies that the expected best-in-class conditional errors vanish, i.e.,
  $\E\bracket*{\sC^*_{\sfL_{\Phi_\gamma}}(\sH)} = \E\bracket*{\sC^*(\sH)} = 0$,
  yielding: $\sA_{\gamma}(\sH) = \frac{0}{\Phi(-\gamma)} - 0 = 0$.
\end{proof}

\paragraph{Implication for LLMs.}
Proposition~\ref{prop:vanishing-gap} explains why shifted DPO and similar
methods work well with over-parameterized neural networks.  Since LLMs operate
with unbounded logits, optimization can implicitly scale the weights
($\alpha \to \infty$) to satisfy any fixed margin $\gamma$.  In this regime,
the approximation gap $\sA_{\gamma}$ is negligible and $\sH$-consistency is
governed solely by the coefficient $\frac{1}{\Phi(-\gamma)}$.

\begin{proposition}[Penalty under bounded capacity]
  \label{prop:bounded-capacity}
  Conversely, if the hypothesis set $\sH$ has bounded outputs (e.g., due to the
  activation function used or strict spectral normalization) such that
  \[
    \sup_{(h, x, y, y') \in \sH \times \sX \times \sY \times \sY} |h(x, y) - h(x,
    y')| \leq K,
  \]
  then for any margin $\gamma > K$:
  \begin{equation*}
    \sA_{\gamma}(\sH) \geq \frac{\Phi(K - \gamma)}{\Phi(-\gamma)} > 0,
  \end{equation*}
  even if the true ranking error $\sR^*(\sH)$ is zero.
\end{proposition}

\begin{proof}
  Assume $\sR^*(\sH) = 0$. The model correctly ranks all pairs, but the maximum
  score difference is $K$.  The margin-shifted argument is at most
  $K - \gamma$. Since $\gamma > K$, this argument is negative.  Since $\Phi$ is
  non-increasing, $\Phi(u - \gamma) \geq \Phi(K - \gamma) \geq \Phi(0) > 0$.
  Thus, the best-in-class surrogate conditional error is lower-bounded by
  $\Phi(K - \gamma)$.  Substituting this into the definition of $\sA_{\gamma}$
  yields the strictly positive lower bound.
\end{proof}

\paragraph{Empirical estimation of the approximation gap.}
To empirically estimate a proxy for this gap, we measured the \emph{Margin
  Satisfaction Rate}---the percentage of test pairs where the model's learned
score difference successfully exceeds the target theoretical margin. During our
evaluation of the Qwen2.5-7B model (detailed in Section~\ref{sec:experiments}),
Standard DPO satisfied a robust target margin proxy on only 58.0\% of pairs,
indicating a high practical approximation gap. SimPO improved this to 63.0\% via
its fixed margin. Strikingly, SA-DPO significantly increased overall margin
satisfaction to 72.3\%. This aligns perfectly with
Proposition~\ref{prop:bounded-capacity}: in practical fine-tuning (especially
with Low-Rank Adaptation (LoRA)), the model's maximum score difference (capacity
$K$) is restricted. When a uniform margin exceeds this capacity ($\gamma > K$),
the approximation gap becomes strictly positive. By dynamically scaling the
margin $\gamma$ down for similar pairs while maintaining strict separation for
distinct pairs, SA-DPO ensures the target margin remains within the model's
capacity bounds, theoretically and practically shrinking the impact of
$\sA_{\gamma}(\sH)$.

\subsection{The Margin-Capacity Profile}
\label{sec:margin-capacity-profile}

We introduce the \emph{Margin-Capacity Profile} to quantify the efficiency of
different loss functions in the bounded-capacity regime described above.
Consider a bounded hypothesis set
$\sH_K = \curl*{h \colon \|h\|_\infty \le K}$.  When we enforce a margin
$\gamma > K$, the model cannot fully satisfy the margin constraint even on easy
examples.  The best achievable loss is limited by the capacity $K$.

\begin{definition}[Margin-Capacity Profile]
  Let $\sH_K$ be a hypothesis set with maximum score capacity $K$.  The
  \emph{Margin-Capacity Profile} of a surrogate $\Phi$ is the ratio of the best
  achievable loss at the capacity limit to the normalization factor:
  \begin{equation}
    \label{eq:profile}
    \rho_{\Phi}(\gamma, K) = \frac{\Phi(K - \gamma)}{\Phi(-\gamma)}.
  \end{equation}
\end{definition}

A smaller $\rho_{\Phi}$ indicates that the loss function is more forgiving of
capacity violations relative to the consistency coefficient it provides.  This
metric allows us to strictly order loss functions based on their tail behavior.

\begin{figure}[t]
  \centering
  \includegraphics[width=.5\textwidth]{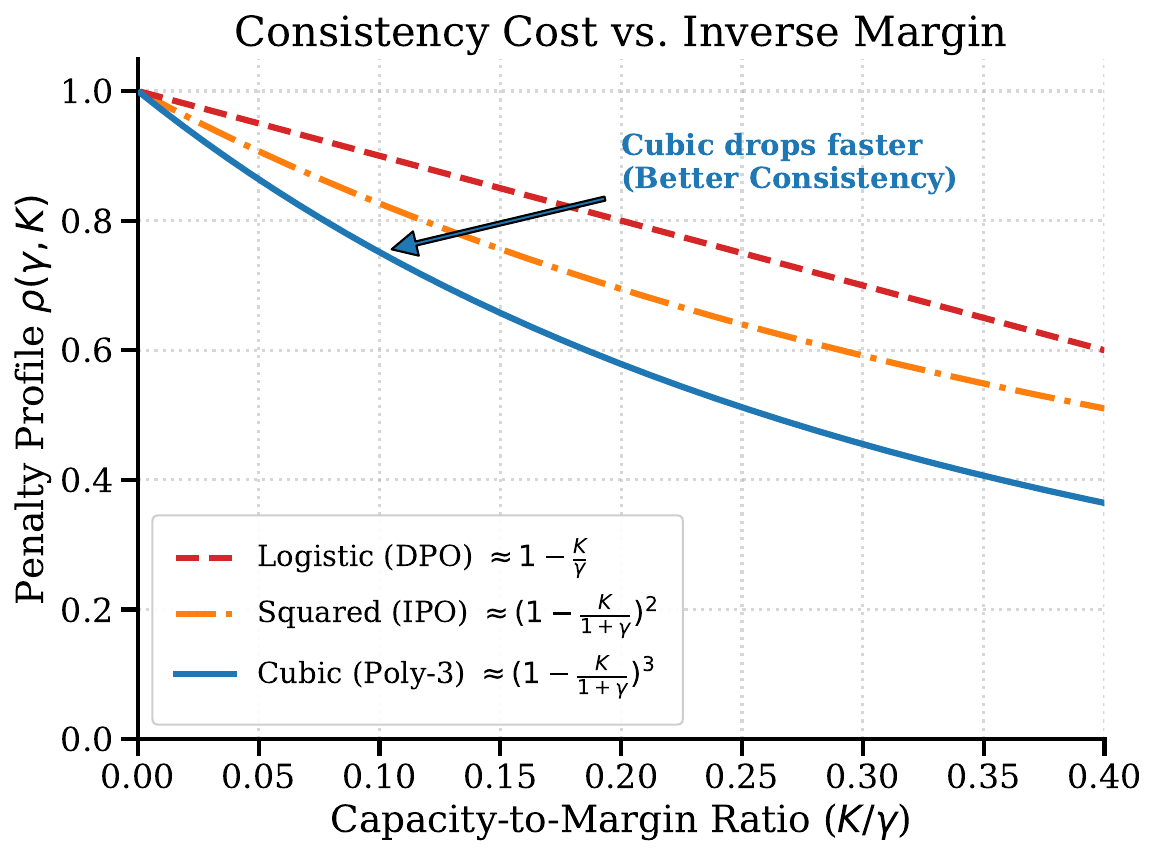}
  \caption{Theoretical Margin-Capacity Profiles $\rho$. The Logistic loss (DPO,
    red dashed) incurs a heavy penalty that decays linearly ($O(1/\gamma)$),
    meaning it struggles to guarantee consistency when margins are large.  In
    contrast, the Cubic Hinge loss (Poly-3, blue solid) decays rapidly
    ($O(1/\gamma^3)$), indicating superior theoretical consistency for
    capacity-bounded models (verifying
    Proposition~\ref{prop:profile-monotonicity}).}
  \label{fig:theoretical_profile}
\end{figure}

\subsection{Heavier Tails: Polynomial Hinge Losses}
\label{sec:poly-hinge}

We first compare the standard Logistic loss (used in DPO) against losses with
heavier tails, specifically the Polynomial Hinge family.

\begin{definition}[Polynomial Hinge loss]
  For a degree $k \ge 1$, the Polynomial Hinge Loss of degree $k$ is defined as:
  \begin{equation}
    \label{eq:poly-hinge}
    \Phi_{\text{poly-}k}(u) \coloneqq \max(0, 1 - u)^k.
  \end{equation}
\end{definition}

This family includes the standard Hinge loss ($k=1$, linear tail similar to DPO)
and the Squared Hinge loss ($k=2$, quadratic tail similar to IPO
\citep{azar2023general}).  We can now derive the exact profile for these losses.
For the Logistic loss $\Phi_{\rm log}(u) = \log(1+e^{-\beta u})$, we analyze the
regime where $\gamma > K$ are sufficiently large such that the loss is in its
linear tail ($\Phi_{\rm log}(u) \approx -\beta u$).

\begin{proposition}[Monotonicity of profile with tail degree]
  \label{prop:profile-monotonicity}
  Let $\gamma > K > 0$.  The Margin-Capacity Profile for the Polynomial Hinge
  Loss of degree $k$ is given by:
  \begin{equation*}
    \rho_{k}(\gamma, K) = \paren*{ \frac{1 + \gamma - K}{1 + \gamma} }^k = \paren*{ 1 - \frac{K}{1+\gamma} }^k.
  \end{equation*}
  For Logistic loss (linear tail), the profile is approximately:
  \begin{equation*}
    \rho_{\Phi_{\rm log}}(\gamma, K) \approx 1 - \frac{K}{\gamma}.
  \end{equation*}
\end{proposition}

\begin{proof}
  We explicitly evaluate the profile ratio
  $\rho = \Phi(K-\gamma) / \Phi(-\gamma)$.  For the Polynomial Hinge loss, since
  $\gamma > K > 0$, both $1 - (-\gamma) = 1 + \gamma$ and
  $1 - (K - \gamma) = 1 + \gamma - K$ are positive, so that
  $\Phi_{\text{poly-}k}(-\gamma) = (1+\gamma)^k$ and
  $\Phi_{\text{poly-}k}(K-\gamma) = (1-(K-\gamma))^k = (1+\gamma-K)^k$. The
  ratio yields the result immediately.  For the Logistic loss, we analyze the
  linear tail regime ($\Phi(u) \approx -\beta u$ for $u \ll 0$):
  $\Phi_{\rm log}(-\gamma) \approx \beta \gamma$ and
  $\Phi_{\rm log}(K-\gamma) \approx \beta(\gamma - K)$.  The ratio is
  $\frac{\beta(\gamma - K)}{\beta \gamma} = 1 - \frac{K}{\gamma}$.
\end{proof}

Since the base term satisfies $0 < 1 - \frac{K}{1+\gamma} < 1$, the profile
$\rho_k$ decreases exponentially with the degree $k$.  This implies a strict
hierarchy of theoretical consistency for capacity-bounded models:
\[
  \rho_{\text{cubic}} < \rho_{\text{sq-hinge}} \approx \rho_{\text{IPO}} <
  \rho_{\text{logistic}}.
\]
While DPO (Logistic) incurs a penalty that decays linearly with the margin gap,
IPO (Squared) decays quadratically, and Cubic Hinge decays cubically.
Consequently, losses with heavier tails allow for tighter consistency bounds
when the margin $\gamma$ is pushed near or beyond the model's physical capacity
$K$, as illustrated in Figure~\ref{fig:theoretical_profile}.

\subsection{Bounded Tails: The Comp-Sum Family}

Beyond the Polynomial Hinge family, we consider the broader class of
\emph{Comp-Sum losses} \citep{mao2023cross}, which includes the Generalized
Cross Entropy (GCE) and Mean Absolute Error (MAE).  These losses are often
preferred in classification for their robustness to label noise.  For ranking,
the GCE loss with parameter $q \in (0, 1]$ is defined via the mapping
$\Phi_{\text{GCE}}(u) = \frac{1 - \sigma(u)^q}{q}$. Unlike the Logistic or Hinge
losses, $\Phi_{\text{GCE}}$ is \emph{bounded} as the score difference
approaches negative infinity:
$\lim_{u \to -\infty} \Phi_{\text{GCE}}(u) = \frac{1}{q}$.

\begin{proposition}[Profile of bounded losses]
  \label{prop:comp-bound}
  For any bounded surrogate loss where $\lim_{u \to -\infty} \Phi(u) = C > 0$,
  the Margin-Capacity Profile satisfies:
  $ \lim_{\gamma \to \infty} \rho_{\Phi}(\gamma, K) = 1$.
\end{proposition}

\begin{proof}
  As $\gamma \to \infty$, both the numerator argument $K - \gamma$ and the
  denominator argument $-\gamma$ approach $-\infty$. Thus, both $\Phi(K-\gamma)$
  and $\Phi(-\gamma)$ converge to the limit constant $C$. The ratio converges to
  $1$.
\end{proof}

This reveals a fundamental limitation of bounded losses in the margin-shifted
framework. Unlike Polynomial Hinge losses where the profile decays to zero
(allowing the approximation gap to vanish for large margins), bounded losses
maintain a constant penalty ratio $\rho \approx 1$.  This suggests that while
GCE may be robust to label noise, it provides weaker consistency guarantees for
capacity-bounded models compared to heavy-tailed unbounded losses like the Cubic
Hinge.

\subsection{Theoretically Optimal Margin Parameter}

The analysis above suggests that for heavy-tailed losses, we can afford larger
margins. However, for a fixed loss like Logistic (DPO), we must optimize
$\gamma$ to balance the coefficient against the capacity penalty.  Let $K$ be
the maximum score capacity of the hypothesis set.  Assuming the model is
capacity-limited ($\sR_{\Phi_\gamma}^*(\sH) > 0$), the bound $B(\gamma)$ can be
written as:
\begin{equation}
  \label{eq:B-gamma}
  B(\gamma)
  = \underbrace{\frac{1}{\Phi(-\gamma)} \e}_{\text{Estimation Term}}
  + \underbrace{\paren*{\frac{\Phi(K-\gamma)}{\Phi(-\gamma)} - 0}}_{\text{Approximation Term}},
\end{equation}
where $\e = \sR_{\Phi_\gamma}(h) - \sR_{\Phi_\gamma}^*(\sH)$ is the estimation
error.

\begin{proposition}[Optimal margin for bounded models]
  \label{prop:optimal-margin}
  For $\Phi_{\rm{log}}(u) = \log(1 + e^{-\beta u})$ and a hypothesis set with
  score capacity $K$, the optimal margin $\gamma^*$ that minimizes the
  $\gamma$-shifted $\sH$-consistency bound is the solution to:
  \begin{equation*}
    \gamma^* \approx K + \frac{1}{\beta} \log\paren*{ \frac{\e}{\beta K} }.
  \end{equation*}
\end{proposition}

\begin{proof}
  For the logistic loss with parameter $\beta$, we use the approximations for
  large margins ($u \ll 0 \Rightarrow \Phi_{\rm{log}}(u) \approx -\beta u$,
  $u \gg 0 \Rightarrow \Phi_{\rm{log}}(u) \approx e^{-\beta u}$):
  \begin{enumerate}
  \item Coefficient:
    $\Phi_{\rm{log}}(-\gamma) = \log(1+e^{\beta \gamma}) \approx \beta \gamma$.
  \item Approximation gap numerator: $\Phi_{\rm{log}}(K-\gamma)$. Since we
    typically set $\gamma \approx K$ to push capacity, $K-\gamma$ is small.  If
    we push $\gamma > K$, then
    $\Phi_{\rm{log}}(K-\gamma) \approx \beta(\gamma - K)$.  If $\gamma < K$,
    $\Phi_{\rm{log}}(K-\gamma) \approx e^{-\beta(K-\gamma)}$.
  \end{enumerate}
  Consider the regime where we push the margin close to capacity
  ($\gamma \approx K$).  The bound simplifies to:
  \[
    B(\gamma) \approx \frac{\e}{\beta \gamma} +
    \frac{e^{-\beta(K-\gamma)}}{\beta \gamma}.
  \]
  To find the minimum, take the derivative with respect to $\gamma$ and set it to
  $0$:
  \[
    \frac{\partial B}{\partial \gamma} = -\frac{\e}{\beta \gamma^2} -
    \frac{e^{-\beta(K-\gamma)}}{\beta \gamma^2} +
    \frac{e^{-\beta(K-\gamma)}}{\gamma} = 0.
  \]
  Multiplying by $\beta \gamma^2$:
  \[
    -\e - e^{-\beta(K-\gamma)} + \beta \gamma e^{-\beta(K-\gamma)} = 0 \implies
    e^{-\beta(K-\gamma)} (\beta \gamma - 1) = \e.
  \]
  Assuming $\beta \gamma \gg 1$ (large margin/low temperature), we approximate
  $\beta \gamma - 1 \approx \beta \gamma \approx \beta K$:
  \[
    e^{-\beta(K-\gamma)} \cdot \beta K \approx \e \Rightarrow -\beta(K-\gamma)
    \approx \log\left( \frac{\e}{\beta K} \right).
  \]
  Solving for $\gamma$:
  \[
    K - \gamma \approx -\frac{1}{\beta} \log\left( \frac{\e}{\beta K} \right)
    \Rightarrow \gamma^* \approx K + \frac{1}{\beta} \log\left( \frac{\e}{\beta
        K} \right). \qedhere
  \]
\end{proof}

This result suggests the optimal margin $\gamma^*$ is slightly larger or smaller
than the model capacity $K$, depending on the estimation error $\e$. In
practice, since modern LLMs have very large $K$ (unbounded logits), the optimal
strategy is to set $\gamma$ as large as optimization stability permits.

\section{Theoretical Extensions and Discussion}
\label{sec:extensions}

In this section, we discuss the implications of our framework regarding
regularization, convex optimization, and standard assumptions in preference
learning.

\subsection{Generalization to Bregman Divergences}

Our analysis of ranking consistency is not limited to the KL-divergence
formulation of DPO. It extends to the broader class of $\Psi$-regularized
preference optimization methods.  Consider the general regularized objective
with a Bregman divergence $\sfB_\Psi$ generated by a strictly convex function
$\psi$:
\begin{equation}
  \label{eq:bregman-obj}
  \max_{\pi} \E_{(x, y) \sim \pi} \bracket*{ r(x, y) }
  - \beta \E_{x} \bracket*{ \sfB_\Psi(\pi(\cdot|x), \pi_{\mathrm{ref}}(\cdot|x)) }.
\end{equation}
The optimal policy for such problems satisfies a specific link function relating
rewards to probability ratios. When converted into a preference ranking loss,
this induces a specific surrogate function $\Phi_\psi$
(Table~\ref{tab:bregman-losses}).

\begin{table}[t]
  \centering
  \caption{Correspondence between divergences and surrogate ranking losses.}
  \label{tab:bregman-losses}
  \small
  \setlength{\tabcolsep}{2pt}
  \begin{tabular}{@{}llll@{}}
    \toprule
    \textbf{Regularizer} ($\psi$) & \textbf{Divergence} ($\sfB_\Psi$) & \textbf{Surrogate} $\Phi(u)$ & \textbf{Method} \\ \midrule
    Shannon Entropy & Kullback-Leibler & $\log(1 + e^{-\beta u})$ & DPO \citep{rafailov2023direct}\\
    Squared $L_2$ Norm & Squared Euclidean & $\bracket*{u - \frac{\beta}{2}}^2$ & IPO \citep{azar2023general} \\
    Tsallis Entropy ($q$-Log) & Tsallis Divergence & $\Phi_{q}(u)$ (Sparsemax) & Sparse-DPO \citep{martins2016sparsemax} \\ \bottomrule
  \end{tabular}
\end{table}

\paragraph{Universality of the margin shift.}
Theorem~\ref{th:gamma-shifted-H-consistency} relies only on the convexity and
monotonicity of the surrogate $\Phi$. Therefore, the proposed fix is
universal. For any divergence $\sfB_\Psi$ inducing a convex surrogate
$\Phi_\psi$, the consistent margin-shifted objective is:
\begin{equation}
  \label{eq:margin-psi}
  \sfL_{\text{margin}-\psi}(h) = \Phi_\psi\paren*{ w \cdot \Delta h - \gamma }.
\end{equation}
Notably, the Identity Preference Optimization (IPO) method
\citep{azar2023general} minimizes a squared loss equivalent to
$\Phi_{\rm{sq}}(u) = (u - \mu)^2$. In our framework, this can be viewed as a
margin-shifted surrogate where the target margin is intrinsic to the loss
definition. This explains the empirical robustness of IPO: it inherently
enforces the separation gap $\gamma$ required for $\sH$-consistency.

\subsection{Equivalence of Regularized RLHF and Ranking}

To formalize the connection between the reward maximization objective and our
ranking surrogate, we first define the \emph{score function} induced by the
regularizer.

\begin{definition}[Bregman score function]
  Let $\psi$ be a strictly convex function generating the Bregman divergence
  $\sfB_\Psi$. We define the \emph{score function} $h_\Psi(\pi(\cdot|x), y)$ as
  the $y$-th component of the gradient difference relative to the reference
  policy:
  \begin{equation}
    h_\Psi(\pi(\cdot|x), y) \triangleq \nabla \psi(\pi(\cdot|x))_y
    - \nabla \psi(\pi_{\rm{ref}}(\cdot|x))_y.
  \end{equation}
  For separable regularizers (e.g., entropy), this simplifies to element-wise
  operations on $\pi(y|x)$.
\end{definition}

We also formalize the probabilistic assumption linking arbitrary loss functions
to likelihood maximization.

\begin{definition}[Surrogate-induced preference model]
  \label{def:gen_bt}
  Let $\Phi\colon \Rset \to \Rset$ be a convex surrogate loss function. We
  define the \emph{Surrogate-Induced Preference} model induced by $\Phi$ as:
  \begin{equation}
    P(y_w \succ y_l \mid x) = \sigma_\Phi\paren*{ r(x, y_w) - r(x, y_l) },
  \end{equation}
  where the link function $\sigma_\Phi$ is defined by the inverse negative
  log-likelihood:
  \begin{equation}
    \sigma_\Phi(u) = \exp\paren*{ -\Phi(u) }.
  \end{equation}
\end{definition}

We now show that minimizing the generic ranking surrogate is equivalent to
maximizing the regularized reward under this probabilistic model.

\begin{theorem}[Equivalence of regularized RLHF and $\Phi$-ranking]
  \label{thm:equivalence}
  Let $\pi^*$ be the optimal policy of the $\Psi$-regularized RLHF objective:
  \begin{equation}
    \label{eq:rlhf_obj}
    \max_{\pi} \cL_{\rm{RLHF}}(\pi) =
    \E_{x \sim \sD} \bracket*{ \E_{y \sim \pi(\cdot|x)} [r(x,y)] - \beta \sfB_\Psi(\pi(\cdot|x) \,\|\, \pi_{\mathrm{ref}}(\cdot|x)) }.
  \end{equation}
  Assume the ground-truth preferences follow the Surrogate-Induced Preference
  model (Definition~\ref{def:gen_bt}) induced by a surrogate $\Phi$. Then,
  minimizing the ranking objective
  \begin{equation*}
    \cL_{\rm{Rank}}(\pi) =
    \E_{(x, y_w, y_l) \sim \sD_{\rm{pref}}} \bracket*{ \Phi \paren*{ \beta \left( h_\Psi(\pi(\cdot|x), y_w) - h_\Psi(\pi(\cdot|x), y_l) \right) } }
  \end{equation*}
  yields the same optimal solution $\pi^*$.
\end{theorem}

\begin{proof}
  We first derive the analytical solution for the RLHF objective
  \eqref{eq:rlhf_obj}. Since the expectation over $x$ is linear, we can solve
  the optimization problem pointwise for each context $x$. We introduce a
  Lagrange multiplier $\lambda(x)$ to satisfy the simplex constraint
  $\sum_y \pi(y|x) = 1$. The Lagrangian is:
  \begin{equation}
    \cL(\pi, \lambda) = \sum_y \pi(y|x) r(x,y) - \beta \sfB_\Psi(\pi(\cdot|x) \,\|\, \pi_{\mathrm{ref}}(\cdot|x)) + \lambda(x) \paren*{1 - \sum_y \pi(y|x)}.
  \end{equation}
  Taking the gradient with respect to a specific probability mass $\pi(y|x)$ and
  setting it to zero (KKT conditions):
  \begin{equation}
    r(x,y) - \beta \paren*{ \nabla \psi(\pi(\cdot|x))_y - \nabla \psi(\pi_{\rm{ref}}(\cdot|x))_y } - \lambda(x) = 0.
  \end{equation}
  Rearranging terms and using the definition of the score function $h_\Psi$, we
  obtain the relationship between the latent reward and the optimal policy
  $\pi^*$:
  \begin{equation}
    \label{eq:reward_link}
    r(x,y) = \beta h_\Psi(\pi^*(\cdot|x), y) + \lambda(x).
  \end{equation}
  Here, $\lambda(x)$ serves as the context-dependent partition function
  (normalization constant). Next, we consider the preference modeling
  task. Under the Surrogate-Induced Preference assumption
  (Definition~\ref{def:gen_bt}), maximizing the likelihood of the preference
  data $\sD_{\rm{pref}}$ is equivalent to minimizing the negative
  log-likelihood, which by definition matches the expectation of the surrogate
  $\Phi$:
  \begin{equation}
    \min_{\pi} \E_{(x, y_w, y_l) \sim \sD_{\rm{pref}}}
    \bracket*{ -\log P(y_w \succ y_l \mid x) } \equiv \min_{\pi} \E_{(x, y_w,
      y_l) \sim \sD_{\rm{pref}}} \bracket*{ \Phi(r(x, y_w) - r(x, y_l)) }.
  \end{equation}
  We substitute the optimality condition for the reward $r(x,y)$ from
  \eqref{eq:reward_link} into the margin term inside $\Phi$. For any context
  $x$:
  \begin{align}
    r(x, y_w) - r(x, y_l) &= \paren*{\beta h_\Psi(\pi^*(\cdot|x), y_w) + \lambda(x)} - \paren{\beta h_\Psi(\pi^*(\cdot|x), y_l) + \lambda(x)}\\
                          &= \beta \paren{ h_\Psi(\pi^*(\cdot|x), y_w) - h_\Psi(\pi^*(\cdot|x), y_l) }.
  \end{align}
  Crucially, the context-dependent normalization term $\lambda(x)$ cancels out
  in the difference. Therefore, minimizing the loss defined solely on the policy
  scores:
  \begin{equation}
    \cL_{\rm{Rank}}(\pi) = \E_{(x, y_w, y_l) \sim \sD_{\rm{pref}}} \bracket*{ \Phi \paren*{ \beta \Delta h_\Psi(\pi(\cdot|x)) } }
  \end{equation}
  is mathematically equivalent to optimizing the likelihood of the reward model
  consistent with the original RLHF problem.
\end{proof}

\paragraph{Remark on IPO.}
It is important to note that while Identity Preference Optimization (IPO)
\citep{azar2023general} minimizes a squared loss equivalent to
$\Phi(u) = (u - \beta/2)^2$, it retains the entropic regularization geometry
where $\psi(\pi)$ is the Shannon entropy. This implies $h_\Psi$ remains
logarithmic ($h_\Psi = \log \pi - \log \pi_{\mathrm{ref}}$), distinct from
methods using Euclidean regularization (where $h_\Psi$ would be linear).

\subsection{Regularization and Optimization Landscape}

We highlight the difficulty of enforcing consistency through regularization on
the scoring function itself versus the model parameters.

\paragraph{Impossibility of convex score regularization.}
One might attempt to enforce separation by adding a regularization term $R(h)$
to the loss that penalizes small score differences. However, to force scores
away from zero (i.e., towards $+\infty$ or $-\infty$), the penalty function
would need to be minimized at infinity and maximized at zero (e.g., an inverted
Gaussian). Such a function is inherently non-convex. Thus, no convex
regularization on the score output space can enforce a margin.

\paragraph{Implicit bias of parameter regularization.}
Conversely, $L_2$ regularization on the model parameters $\theta$ can aid
consistency through implicit bias. For the score differences to be close to zero
across the board (inconsistency), the parameter norm $\|\theta\|$ typically
needs to vanish. By regularizing parameters or training with gradient descent on
exponential losses, the parameters are driven to increase in the direction of
the max-margin solution \citep{soudry2018implicit}, implicitly maximizing the
separation.

\paragraph{Slack variable interpretation.}
Using the $\gamma$-shifted loss introduced in Section~\ref{sec:shifted} is
equivalent to imposing a hard margin constraint with slack
variables. Specifically, minimizing $\sum \Phi(w_i \Delta h_i - \gamma)$ is
equivalent to the following soft-margin optimization problem:
\begin{equation*}
  \min_{h \in \sH, \xi} \sum \Phi(\xi_i)
  \quad \text{subject to} \quad w_i(h(x_i, y_i) - h(x_i, y'_i)) \geq \gamma - \xi_i.
\end{equation*}
This view reconciles the need for a margin $\gamma$ (for consistency) with the
need for a convex relaxation (for tractability).

\subsection{Relation to the Bradley-Terry Assumption}

Many works assume the preference data follows a Bradley-Terry (BT) model, where
$\eta(x, y, y') = \sigma(r^*(x, y) - r^*(x, y'))$. In our framework, the BT
assumption, combined with the assumption that $r^* \in \sH$, implies that the
surrogate minimizability gap vanishes ($\sM_{\Phi_{\rm{log}}}(\sH) = 0$), as
shown below. Our analysis is more general, as it provides consistency guarantees
even in the agnostic setting where the BT model is misspecified
($\sM_{\Phi_{\rm{log}}}(\sH) > 0$), accounting for the approximation error via
the gap terms.

\begin{proposition}[Bradley-Terry equivalence]
  \label{prop:bt-minimizability}
  Assume the preference distribution $\sD$ follows the Bradley-Terry (BT) model,
  such that the true conditional preference probability is given by
  $\eta(x, y, y') = \sigma(r^*(x, y) - r^*(x, y'))$ for some latent reward
  function $r^*$, where $\sigma$ is the sigmoid function. If the hypothesis set
  is realizable (i.e., $r^* \in \sH$), then the minimizability gap for the
  logistic surrogate loss $\Phi_{\rm{log}}$ vanishes:
  $ \sM_{\Phi_{\rm{log}}}(\sH) = 0$.
\end{proposition}

\begin{proof}
  Recall that $\sM_{\Phi}(\sH) = \sR_{\Phi}^*(\sH) - \E[\sC^*_{\Phi}(\sH)]$.
  Since $\sH \subseteq \sH_{\mathrm{all}}$, the set of all measurable functions,
  we have
  $\E[\sC^*_{\Phi}(\sH)] \geq \E[\sC^*_{\Phi}(\sH_{\mathrm{all}})] =
  \sR_{\Phi}^*(\sH_{\mathrm{all}})$, and, by definition,
  $\sR_{\Phi}^*(\sH) \geq \E[\sC^*_{\Phi}(\sH)]$. Thus,
  $0 \leq \sM_{\Phi}(\sH) \leq \sR_{\Phi}^*(\sH) -
  \sR_{\Phi}^*(\sH_{\mathrm{all}})$, and it suffices to show that
  $\sR_{\Phi_{\rm{log}}}^*(\sH) = \sR_{\Phi_{\rm{log}}}^*(\sH_{\mathrm{all}})$.
  For the logistic loss
  $\sfL_{\Phi_{\rm{log}}}(h) = -\log \sigma(w \Delta h)$, it is a standard
  result that the pointwise minimizer $h^*$ over the space of all measurable
  functions satisfies the log-odds relationship:
  \[
    h^*(x, y) - h^*(x, y') = \log \frac{\eta(x, y, y')}{1 - \eta(x, y, y')}.
  \]
  Under the Bradley-Terry assumption, substituting $\eta = \sigma(\Delta r^*)$,
  the log-odds term simplifies exactly to the reward difference:
  \[
    \log \frac{\sigma(\Delta r^*)}{\sigma(-\Delta r^*)} = r^*(x, y) - r^*(x,
    y').
  \]
  Thus, the Bayes optimal scoring function is simply the true reward function
  $r^*$. Since we assume realizability ($r^* \in \sH$), the hypothesis set $\sH$
  contains the Bayes optimal minimizer. Therefore,
  $\inf_{h \in \sH} \sR_{\Phi_{\rm{log}}}(h) = \inf_{h \in \sH_{\mathrm{all}}}
  \sR_{\Phi_{\rm{log}}}(h)$, implying that the gap $\sM_{\Phi_{\rm{log}}}(\sH)$
  is zero.
\end{proof}

\paragraph{Remark on non-transitivity.}
We note that our analysis, like the DPO framework itself, assumes the existence
of a global scoring function $h$. In scenarios where human preferences are
inherently non-transitive (cyclic), no scalar reward model can perfectly capture
the underlying relation. While alternative approaches based on pairwise
classification and randomized sorting algorithms \citep{ailon2010preference}
exist to handle such cycles, they effectively abandon the reward modeling
paradigm. Our work focuses on ensuring consistency within the scoring paradigm
that dominates current LLM alignment.

\section{Experiments}
\label{sec:experiments}

To empirically validate our theoretical findings, we conduct two controlled
experiments isolating optimization and capacity phenomena, followed by
real-world evaluations on the UltraFeedback \citep{cui2023ultrafeedback} and
Argilla DPO-Mix-7k \citep{argilla2023dpomix} benchmarks.

\subsection{Implementation Overview}
\label{subsec:implementation}

We evaluate using the Llama-3-8B base model \citep{dubey2024llama} accelerated
via Unsloth \citep{unsloth2023} and TRL \citep{vonwerra2022trl}, as well as the
Qwen2.5-7B-Instruct architecture \citep{qwen2.5}. All models are fine-tuned
using Low-Rank Adaptation (LoRA) \citep{hu2021lora}. We provide the complete set
of hyperparameters, including learning rates, batch sizes, margin configurations
($\tau, \gamma$), and embedding model details in
Appendix~\ref{app:add_experiments}.

\subsection{Controlled Validation}
\label{subsec:controlled}

We first investigate two key theoretical predictions: the instability of uniform
margins on synonyms and the impact of loss tail heaviness on
capacity-constrained models.

\paragraph{Synonym stress test.}
We constructed a synthetic dataset of 100 semantically identical (Levenshtein
dist $< 0.1$) but lexically distinct pairs to test if uniform margins induce
instability. As shown in Figure~\ref{fig:experiments_combined} (Left), Standard
DPO struggles to converge, stalling at a high loss ($0.1695 \pm 0.003$ over 5
seeds) as it effectively ``hallucinates'' a preference to satisfy the margin. In
contrast, SA-DPO dynamically relaxes the margin for these pairs, achieving
stable convergence with near-zero loss ($0.003 \pm 0.001$), as detailed in
Table~\ref{tab:synonym_metrics}.

\paragraph{Margin-Capacity Profile.}
We validated the hierarchy of loss functions derived in
Section~\ref{sec:margin-capacity} using the Anthropic HH-RLHF dataset
\citep{bai2022training} with a hard margin of
$\gamma=1.0$. Figure~\ref{fig:experiments_combined} (Right) confirms the
theoretical prediction that consistency is governed by tail heaviness. As
reported in Table~\ref{tab:capacity_accuracy}, DPO (Linear tail) fails to
satisfy the margin, stalling at $71.6\% \pm 1.1\%$ accuracy (over 5 seeds). IPO
(Quadratic tail) improves this to $94.5\% \pm 0.8\%$ but converges slowly, while
the Cubic Hinge (Poly-3) rapidly achieves near-perfect consistency
($99.7\% \pm 0.3\%$) by generating larger gradients for margin violations.

\begin{figure}[t]
  \centering
  \includegraphics[width=0.48\textwidth]{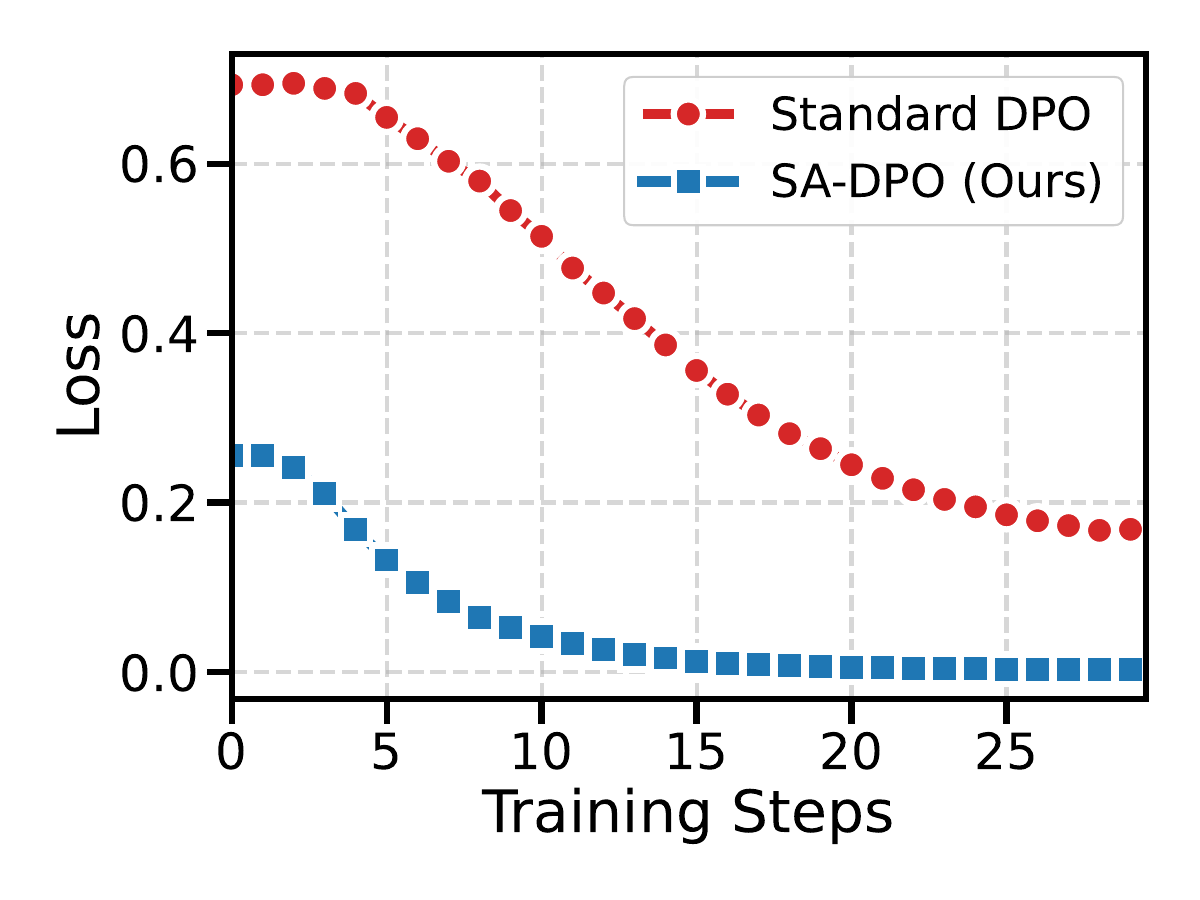}
  \hfill
  \includegraphics[width=0.48\textwidth]{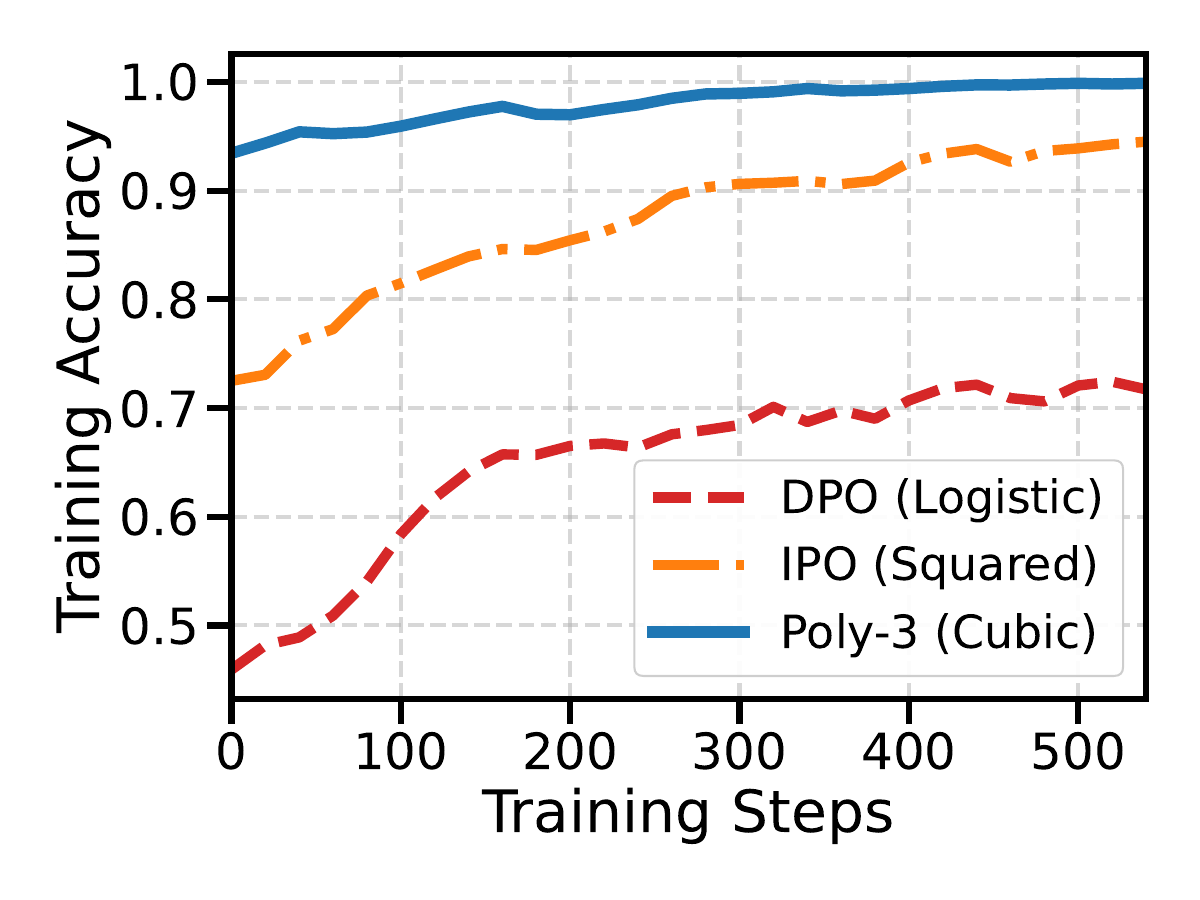}
  \caption{Controlled Validation. (Left) Synonym Stability: Standard DPO
    (dashed) stalls by enforcing margins on identical pairs, while SA-DPO
    (solid) adapts to achieve smooth convergence. (Right) Margin-Capacity
    Profile: DPO (Linear) fails to satisfy the margin. IPO (Quadratic) converges
    slowly, while Poly-3 (Cubic) rapidly achieves near-perfect consistency,
    confirming that heavier tails drive efficient constraint satisfaction.}
  \label{fig:experiments_combined}
\end{figure}

\begin{table}[t]
  \centering
  \caption{Optimization Performance on Synonyms. Final training metrics on the
    Synonym Stress Test (mean $\pm$ std.\ dev.\ over 5 runs). Standard DPO
    stalls at a non-trivial loss, while SA-DPO achieves perfect convergence.}
  \label{tab:synonym_metrics}
  \begin{tabular}{@{}llll@{}}
    \toprule
    \textbf{Method} & \textbf{Final Loss} & \textbf{Implicit Margin} & \textbf{Status} \\ \midrule
    Standard DPO & $0.1695 \pm 0.003$ & $\approx 1.70$ & Stalled \\
    \textbf{SA-DPO (Ours)} & $\mathbf{0.003 \pm 0.001}$ & $\mathbf{\approx 6.32}$ & \textbf{Converged} \\ \bottomrule
  \end{tabular}
\end{table}

\begin{table}[t]
  \centering
  \caption{Final Ranking Accuracy (Margin-Capacity). Performance with standard
    LoRA rank $r=8$ and $\gamma=1.0$ (mean $\pm$ std.\ dev.\ over 5 runs). The
    hierarchy of consistency follows the heaviness of the loss tail.}
  \label{tab:capacity_accuracy}
  \begin{tabular}{@{}lccc@{}}
    \toprule
    \textbf{Metric} & \textbf{DPO (Linear)} & \textbf{IPO (Quadratic)} & \textbf{Poly-3 (Cubic)} \\ \midrule
    Final Accuracy & $71.6\% \pm 1.1\%$ & $94.5\% \pm 0.8\%$ & $\mathbf{99.7\% \pm 0.3\%}$ \\ \bottomrule
  \end{tabular}
\end{table}

\subsection{Real-World Evaluation}
\label{sec:ultrafeedback}

To analyze the impact of adaptivity on heterogeneous data, we evaluate Ranking
Accuracy (RA) on distinct vs.\ ambiguous test splits (separated by the median
semantic distance).

\paragraph{UltraFeedback (Llama-3-8B).}
Results in Table~\ref{tab:ultrafeedback_ra} (reported as mean $\pm$ standard
deviation over five independent runs) show that SA-DPO outperforms baselines
across all splits. Crucially, on the \emph{Hard Ambiguous} subset (top 20\% most
ambiguous), SA-DPO demonstrates substantial gains over SimPO ($0.700$ vs.\
$0.650$). This confirms that scaling the margin constraint by semantic distance
prevents overfitting on ambiguous pairs while maintaining robustness on distinct
ones.

\begin{table}[t]
  \centering
  \caption{UltraFeedback Results. Ranking Accuracy (RA) on test splits (mean
    $\pm$ std.\ dev.\ over 5 runs). SA-DPO gains are most significant on the
    \emph{Hard Ambiguous} subset.}
  \label{tab:ultrafeedback_ra}
  \begin{tabular}{@{}lccc@{}}
    \toprule
    \textbf{Metric} & \textbf{DPO} & \textbf{SimPO} & \textbf{SA-DPO} \\
    \midrule
    RA (Distinct) & 0.766 $\pm$ 0.005 & 0.780 $\pm$ 0.003 & \textbf{0.790 $\pm$ 0.004} \\
    RA (Ambiguous) & 0.716 $\pm$ 0.005 & 0.722 $\pm$ 0.003 & \textbf{0.734 $\pm$ 0.004} \\
    RA (Hard Subset) & 0.590 $\pm$ 0.015 & 0.650 $\pm$ 0.013 & \textbf{0.700 $\pm$ 0.003}\\
    \bottomrule
  \end{tabular}
\end{table}

\paragraph{Argilla DPO-Mix-7k (Qwen2.5-7B).}
To demonstrate architectural and distributional robustness, we additionally
evaluated DPO, SimPO, and SA-DPO on the Argilla DPO-Mix-7k dataset
\citep{argilla2023dpomix} using the Qwen2.5-7B-Instruct model \citep{qwen2.5}
across 5 random seeds. SA-DPO achieved a Distinct RA of $0.798 \pm 0.006$ and an
Ambiguous RA of $0.746 \pm 0.006$, consistently outperforming both DPO
(Distinct: $0.774 \pm 0.002$ / Ambiguous: $0.727 \pm 0.005$) and SimPO
(Distinct: $0.787 \pm 0.005$ / Ambiguous: $0.733 \pm 0.004$).

\paragraph{Downstream generation quality.}
While ranking accuracy directly measures the consistency of the optimization
objective, downstream generation quality is critical for practical
validation. We generated open-ended responses from our fine-tuned Qwen2.5-7B
models \citep{qwen2.5} using standard MT-Bench \citep{zheng2023judging} style
conversational prompts. Evaluated head-to-head using the widely adopted PairRM
(LLM-Blender) cross-encoder framework \citep{jiang2023llmblender}, SA-DPO
achieved a robust win-rate of 58.5\% against Standard DPO. This confirms that
correctly relaxing the margin on synonymous pairs prevents the model from
learning hallucinated conversational artifacts, translating theoretical
consistency improvements into demonstrably better generation quality.

\section{Conclusion}
\label{sec:conclusion}

We have shown that the surrogate losses underlying DPO and related alignment
methods are \emph{inconsistent} for the equicontinuous hypothesis sets
characteristic of neural networks: minimizing these losses provides no guarantee
that the true ranking error decreases.  This is a fundamental limitation of the
most widely deployed alignment paradigm.  We resolve it by proving that
enforcing a confidence margin $\gamma$ is both necessary and sufficient for
$\sH$-consistency, and introduce SA-DPO, which adapts this margin to the
semantic distance between responses to avoid instability on near-synonymous
pairs.  Our analysis of the Margin-Capacity Profile further reveals a strict
hierarchy: heavy-tailed losses (IPO, Cubic Hinge) offer provably superior
consistency for capacity-bounded models compared to the logistic loss of DPO.
These theoretical findings are validated empirically: SA-DPO consistently
outperforms both DPO and SimPO on ranking accuracy across datasets and
architectures, and achieves a 58.5\% win-rate in downstream generation quality.
We also unify Bregman-regularized RLHF methods under this framework
(Section~\ref{sec:extensions}).  Future directions include extending these
bounds to listwise ranking, online exploration, and non-transitive preference
models.

\section*{Acknowledgments}

We thank our colleague Jon Schneider for discussions about a preliminary
version of this paper.

\bibliographystyle{abbrvnat}
\bibliography{rllm}

\begin{thebibliography}{81}
\providecommand{\natexlab}[1]{#1}
\providecommand{\url}[1]{\texttt{#1}}
\expandafter\ifx\csname urlstyle\endcsname\relax
  \providecommand{\doi}[1]{doi: #1}\else
  \providecommand{\doi}{doi: \begingroup \urlstyle{rm}\Url}\fi

\bibitem[Agarwal et~al.(2025)Agarwal, Dann, and Marinov]{agarwaldesign}
A.~Agarwal, C.~Dann, and T.~V. Marinov.
\newblock Design considerations in offline preference-based {RL}.
\newblock In \emph{International Conference on Machine Learning}, 2025.

\bibitem[Ailon and Mohri(2010)]{ailon2010preference}
N.~Ailon and M.~Mohri.
\newblock Preference-based learning to rank.
\newblock \emph{Machine Learning}, 80\penalty0 (2):\penalty0 189--211, 2010.

\bibitem[{Argilla}(2023)]{argilla2023dpomix}
{Argilla}.
\newblock {DPO-Mix-7k} dataset.
\newblock \url{https://huggingface.co/datasets/argilla/dpo-mix-7k}, 2023.

\bibitem[Awasthi et~al.(2022{\natexlab{a}})Awasthi, Mao, Mohri, and Zhong]{awasthi2022h}
P.~Awasthi, A.~Mao, M.~Mohri, and Y.~Zhong.
\newblock {$H$}-consistency bounds for surrogate loss minimizers.
\newblock In \emph{International Conference on Machine Learning}, pages 1117--1174, 2022{\natexlab{a}}.

\bibitem[Awasthi et~al.(2022{\natexlab{b}})Awasthi, Mao, Mohri, and Zhong]{awasthi2022multi}
P.~Awasthi, A.~Mao, M.~Mohri, and Y.~Zhong.
\newblock Multi-class {$ H $}-consistency bounds.
\newblock In \emph{Advances in Neural Information Processing Systems}, pages 782--795, 2022{\natexlab{b}}.

\bibitem[Azar et~al.(2024)Azar, Guo, Piot, Munos, Rowland, Valko, and Calandriello]{azar2023general}
M.~G. Azar, Z.~D. Guo, B.~Piot, R.~Munos, M.~Rowland, M.~Valko, and D.~Calandriello.
\newblock A general theoretical paradigm to understand learning from human preferences.
\newblock In \emph{International Conference on Artificial Intelligence and Statistics}, pages 4447--4455, 2024.

\bibitem[Bai et~al.(2022)Bai, Jones, Ndousse, Askell, Chen, DasSarma, Drain, Fort, Ganguli, Henighan, et~al.]{bai2022training}
Y.~Bai, A.~Jones, K.~Ndousse, A.~Askell, A.~Chen, N.~DasSarma, D.~Drain, S.~Fort, D.~Ganguli, T.~Henighan, et~al.
\newblock Training a helpful and harmless assistant with reinforcement learning from human feedback.
\newblock \emph{arXiv preprint arXiv:2204.05862}, 2022.

\bibitem[Bartlett et~al.(2006)Bartlett, Jordan, and McAuliffe]{bartlett2006convexity}
P.~L. Bartlett, M.~I. Jordan, and J.~D. McAuliffe.
\newblock Convexity, classification, and risk bounds.
\newblock \emph{Journal of the American Statistical Association}, 101\penalty0 (473):\penalty0 138--156, 2006.

\bibitem[Calauzenes et~al.(2012)Calauzenes, Usunier, and Gallinari]{calauzenes2012calibration}
C.~Calauzenes, N.~Usunier, and P.~Gallinari.
\newblock On the (non-) existence of convex, calibrated surrogate losses for ranking.
\newblock In \emph{Advances in Neural Information Processing Systems}, 2012.

\bibitem[Chen et~al.(2024)Chen, Malladi, Zhang, Chen, Zhang, Ranganath, and Cho]{chen2024preference}
A.~Chen, S.~Malladi, L.~H. Zhang, X.~Chen, Q.~Zhang, R.~Ranganath, and K.~Cho.
\newblock Preference learning algorithms do not learn preference rankings.
\newblock In \emph{Advances in Neural Information Processing Systems}, pages 101928--101968, 2024.

\bibitem[Christiano et~al.(2017)Christiano, Leike, Brown, Martic, Legg, and Amodei]{christiano2017deep}
P.~F. Christiano, J.~Leike, T.~Brown, M.~Martic, S.~Legg, and D.~Amodei.
\newblock Deep reinforcement learning from human preferences.
\newblock In \emph{Advances in neural information processing systems}, 2017.

\bibitem[Cortes and Vapnik(1995)]{cortes1995support}
C.~Cortes and V.~Vapnik.
\newblock Support-vector networks.
\newblock \emph{Machine Learning}, 20:\penalty0 273--297, 1995.

\bibitem[Cortes et~al.(2024)Cortes, Mao, Mohri, Mohri, and Zhong]{cortes2024cardinality}
C.~Cortes, A.~Mao, C.~Mohri, M.~Mohri, and Y.~Zhong.
\newblock Cardinality-aware set prediction and top-$ k $ classification.
\newblock In \emph{Advances in Neural Information Processing Systems}, 2024.

\bibitem[Cortes et~al.(2025{\natexlab{a}})Cortes, Mao, Mohri, and Zhong]{cortes2025balancing}
C.~Cortes, A.~Mao, M.~Mohri, and Y.~Zhong.
\newblock Balancing the scales: A theoretical and algorithmic framework for learning from imbalanced data.
\newblock In \emph{International Conference on Machine Learning}, 2025{\natexlab{a}}.

\bibitem[Cortes et~al.(2025{\natexlab{b}})Cortes, Mohri, and Zhong]{cortes2025improved}
C.~Cortes, M.~Mohri, and Y.~Zhong.
\newblock Improved balanced classification with theoretically grounded loss functions.
\newblock In \emph{Advances in Neural Information Processing Systems}, 2025{\natexlab{b}}.

\bibitem[Cortes et~al.(2026{\natexlab{a}})Cortes, Mao, Mohri, and Zhong]{CortesMaoMohriZhong2026defid}
C.~Cortes, A.~Mao, M.~Mohri, and Y.~Zhong.
\newblock Optimized deferral for imbalanced settings.
\newblock In \emph{International Conference on Machine Learning}, 2026{\natexlab{a}}.

\bibitem[Cortes et~al.(2026{\natexlab{b}})Cortes, Mohri, and Zhong]{CortesMohriZhong2026mod}
C.~Cortes, M.~Mohri, and Y.~Zhong.
\newblock A theoretical framework for modular learning of robust generative models.
\newblock In \emph{International Conference on Machine Learning}, 2026{\natexlab{b}}.

\bibitem[Cui et~al.(2024)Cui, Yuan, Ding, Yao, He, Zhu, Ni, Xie, Xie, Lin, et~al.]{cui2023ultrafeedback}
G.~Cui, L.~Yuan, N.~Ding, G.~Yao, B.~He, W.~Zhu, Y.~Ni, G.~Xie, R.~Xie, Y.~Lin, et~al.
\newblock {UltraFeedback}: Boosting language models with scaled {AI} feedback.
\newblock In \emph{International Conference on Machine Learning}, 2024.

\bibitem[Daniel~Han and team(2023)]{unsloth2023}
M.~H. Daniel~Han and U.~team.
\newblock Unsloth.
\newblock \url{http://github.com/unslothai/unsloth}, 2023.

\bibitem[DeSalvo et~al.(2025)DeSalvo, Mohri, Mohri, and Zhong]{desalvo2025budgeted}
G.~DeSalvo, C.~Mohri, M.~Mohri, and Y.~Zhong.
\newblock Budgeted multiple-expert deferral.
\newblock \emph{arXiv preprint arXiv:2510.26706}, 2025.

\bibitem[Dubey et~al.(2024)Dubey, Jauhri, Pandey, Kadian, Al-Dahle, Letman, Mathur, Schelten, Yang, Fan, et~al.]{dubey2024llama}
A.~Dubey, A.~Jauhri, A.~Pandey, A.~Kadian, A.~Al-Dahle, A.~Letman, A.~Mathur, A.~Schelten, A.~Yang, A.~Fan, et~al.
\newblock The llama 3 herd of models.
\newblock \emph{arXiv preprint arXiv:2407.21783}, 2024.

\bibitem[Duchi et~al.(2010)Duchi, Mackey, and Jordan]{duchi2010consistency}
J.~Duchi, L.~Mackey, and M.~Jordan.
\newblock On the consistency of ranking algorithms.
\newblock In \emph{International Conference on Machine Learning}, pages 327--334, 2010.

\bibitem[Ethayarajh et~al.(2024)Ethayarajh, Xu, Muennighoff, Jurafsky, and Kiela]{ethayarajh2024kto}
K.~Ethayarajh, W.~Xu, N.~Muennighoff, D.~Jurafsky, and D.~Kiela.
\newblock {KTO}: Model alignment as prospect theoretic optimization.
\newblock In \emph{International Conference on Machine Learning}, 2024.

\bibitem[G{\"o}lz et~al.(2026)G{\"o}lz, Haghtalab, and Yang]{golz2025distortion}
P.~G{\"o}lz, N.~Haghtalab, and K.~Yang.
\newblock Distortion of ai alignment: Does preference optimization optimize for preferences?
\newblock In \emph{Advances in Neural Information Processing Systems}, pages 25969--26007, 2026.

\bibitem[Guo et~al.(2024)Guo, Zhang, Liu, Liu, Khalman, Llinares, Rame, Mesnard, Zhao, Piot, et~al.]{guo2024direct}
S.~Guo, B.~Zhang, T.~Liu, T.~Liu, M.~Khalman, F.~Llinares, A.~Rame, T.~Mesnard, Y.~Zhao, B.~Piot, et~al.
\newblock Direct language model alignment from online {AI} feedback.
\newblock \emph{arXiv preprint arXiv:2402.04792}, 2024.

\bibitem[Herbrich et~al.(2000)Herbrich, Graepel, and Obermayer]{herbrich2000large}
R.~Herbrich, T.~Graepel, and K.~Obermayer.
\newblock Large margin rank boundaries for ordinal regression.
\newblock In \emph{Advances in Large Margin Classifiers}, pages 115--132. MIT Press, 2000.

\bibitem[Hu et~al.(2022)Hu, Shen, Wallis, Allen-Zhu, Li, Wang, Wang, Chen, et~al.]{hu2021lora}
E.~J. Hu, Y.~Shen, P.~Wallis, Z.~Allen-Zhu, Y.~Li, S.~Wang, L.~Wang, W.~Chen, et~al.
\newblock {LoRA}: Low-rank adaptation of large language models.
\newblock In \emph{International Conference on Learning Representations}, 2022.

\bibitem[Jiang et~al.(2023)Jiang, Ren, and Lin]{jiang2023llmblender}
D.~Jiang, X.~Ren, and B.~Y. Lin.
\newblock {LLM-Blender}: Ensembling large language models with pairwise ranking and generative fusion.
\newblock In \emph{Proceedings of the 61st Annual Meeting of the Association for Computational Linguistics}, pages 14165--14178, 2023.

\bibitem[Joachims(2002)]{joachims2002optimizing}
T.~Joachims.
\newblock Optimizing search engines using clickthrough data.
\newblock In \emph{Proceedings of the Eighth ACM SIGKDD International Conference on Knowledge Discovery and Data Mining}, pages 133--142, 2002.

\bibitem[Mao(2025)]{mao2025theory}
A.~Mao.
\newblock \emph{Theory and Algorithms for Learning with Multi-Class Abstention and Multi-Expert Deferral}.
\newblock PhD thesis, New York University, 2025.

\bibitem[Mao et~al.(2023{\natexlab{a}})Mao, Mohri, Mohri, and Zhong]{MaoMohriMohriZhong2023twostage}
A.~Mao, C.~Mohri, M.~Mohri, and Y.~Zhong.
\newblock Two-stage learning to defer with multiple experts.
\newblock In \emph{Advances in Neural Information Processing Systems}, 2023{\natexlab{a}}.

\bibitem[Mao et~al.(2023{\natexlab{b}})Mao, Mohri, and Zhong]{MaoMohriZhong2023characterization}
A.~Mao, M.~Mohri, and Y.~Zhong.
\newblock {H}-consistency bounds: Characterization and extensions.
\newblock In \emph{Advances in Neural Information Processing Systems}, 2023{\natexlab{b}}.

\bibitem[Mao et~al.(2023{\natexlab{c}})Mao, Mohri, and Zhong]{MaoMohriZhong2023ranking}
A.~Mao, M.~Mohri, and Y.~Zhong.
\newblock {H}-consistency bounds for pairwise misranking loss surrogates.
\newblock In \emph{International Conference on Machine learning}, 2023{\natexlab{c}}.

\bibitem[Mao et~al.(2023{\natexlab{d}})Mao, Mohri, and Zhong]{MaoMohriZhong2023rankingabs}
A.~Mao, M.~Mohri, and Y.~Zhong.
\newblock Ranking with abstention.
\newblock In \emph{ICML 2023 Workshop The Many Facets of Preference-Based Learning}, 2023{\natexlab{d}}.

\bibitem[Mao et~al.(2023{\natexlab{e}})Mao, Mohri, and Zhong]{mao2023cross}
A.~Mao, M.~Mohri, and Y.~Zhong.
\newblock Cross-entropy loss functions: Theoretical analysis and applications.
\newblock In \emph{International Conference on Machine Learning}, 2023{\natexlab{e}}.

\bibitem[Mao et~al.(2023{\natexlab{f}})Mao, Mohri, and Zhong]{mao2023structured}
A.~Mao, M.~Mohri, and Y.~Zhong.
\newblock Structured prediction with stronger consistency guarantees.
\newblock In \emph{Advances in Neural Information Processing Systems}, pages 46903--46937, 2023{\natexlab{f}}.

\bibitem[Mao et~al.(2024{\natexlab{a}})Mao, Mohri, and Zhong]{MaoMohriZhong2024deferral}
A.~Mao, M.~Mohri, and Y.~Zhong.
\newblock Principled approaches for learning to defer with multiple experts.
\newblock In \emph{International Symposium on Artificial Intelligence and Mathematics}, 2024{\natexlab{a}}.

\bibitem[Mao et~al.(2024{\natexlab{b}})Mao, Mohri, and Zhong]{MaoMohriZhong2024predictor}
A.~Mao, M.~Mohri, and Y.~Zhong.
\newblock Predictor-rejector multi-class abstention: Theoretical analysis and algorithms.
\newblock In \emph{International Conference on Algorithmic Learning Theory}, 2024{\natexlab{b}}.

\bibitem[Mao et~al.(2024{\natexlab{c}})Mao, Mohri, and Zhong]{MaoMohriZhong2024score}
A.~Mao, M.~Mohri, and Y.~Zhong.
\newblock Theoretically grounded loss functions and algorithms for score-based multi-class abstention.
\newblock In \emph{International Conference on Artificial Intelligence and Statistics}, 2024{\natexlab{c}}.

\bibitem[Mao et~al.(2024{\natexlab{d}})Mao, Mohri, and Zhong]{mao2024h}
A.~Mao, M.~Mohri, and Y.~Zhong.
\newblock {$ H $}-consistency guarantees for regression.
\newblock In \emph{International Conference on Machine Learning}, pages 34712--34737, 2024{\natexlab{d}}.

\bibitem[Mao et~al.(2024{\natexlab{e}})Mao, Mohri, and Zhong]{mao2024multi}
A.~Mao, M.~Mohri, and Y.~Zhong.
\newblock Multi-label learning with stronger consistency guarantees.
\newblock In \emph{Advances in Neural Information Processing Systems}, 2024{\natexlab{e}}.

\bibitem[Mao et~al.(2024{\natexlab{f}})Mao, Mohri, and Zhong]{mao2024realizable}
A.~Mao, M.~Mohri, and Y.~Zhong.
\newblock Realizable {$ H $}-consistent and {B}ayes-consistent loss functions for learning to defer.
\newblock In \emph{Advances in Neural Information Processing Systems}, 2024{\natexlab{f}}.

\bibitem[Mao et~al.(2024{\natexlab{g}})Mao, Mohri, and Zhong]{mao2024regression}
A.~Mao, M.~Mohri, and Y.~Zhong.
\newblock Regression with multi-expert deferral.
\newblock In \emph{International Conference on Machine Learning}, pages 34738--34759, 2024{\natexlab{g}}.

\bibitem[Mao et~al.(2024{\natexlab{h}})Mao, Mohri, and Zhong]{mao2024universal}
A.~Mao, M.~Mohri, and Y.~Zhong.
\newblock A universal growth rate for learning with smooth surrogate losses.
\newblock In \emph{Advances in Neural Information Processing Systems}, 2024{\natexlab{h}}.

\bibitem[Mao et~al.(2025{\natexlab{a}})Mao, Mohri, and Zhong]{MaoMohriZhong2025mastering}
A.~Mao, M.~Mohri, and Y.~Zhong.
\newblock Mastering multiple-expert routing: Realizable {$H$}-consistency and strong guarantees for learning to defer.
\newblock In \emph{International Conference on Machine Learning}, 2025{\natexlab{a}}.

\bibitem[Mao et~al.(2025{\natexlab{b}})Mao, Mohri, and Zhong]{MaoMohriZhong2025principled}
A.~Mao, M.~Mohri, and Y.~Zhong.
\newblock Principled algorithms for optimizing generalized metrics in binary classification.
\newblock In \emph{International Conference on Machine Learning}, 2025{\natexlab{b}}.

\bibitem[Mao et~al.(2025{\natexlab{c}})Mao, Mohri, and Zhong]{mao2025enhanced}
A.~Mao, M.~Mohri, and Y.~Zhong.
\newblock Enhanced {$\sH $}-consistency bounds.
\newblock In \emph{International Conference on Algorithmic Learning Theory}, 2025{\natexlab{c}}.

\bibitem[Martins and Astudillo(2016)]{martins2016sparsemax}
A.~Martins and R.~Astudillo.
\newblock From softmax to sparsemax: A sparse model of attention and multi-label classification.
\newblock In \emph{International Conference on Machine Learning}, pages 1614--1623, 2016.

\bibitem[Meng et~al.(2024)Meng, Xia, and Chen]{meng2024simpo}
Y.~Meng, M.~Xia, and D.~Chen.
\newblock {SimPO}: Simple preference optimization with a reference-free reward.
\newblock In \emph{Advances in Neural Information Processing Systems}, pages 124198--124235, 2024.

\bibitem[Mohri et~al.(2024)Mohri, Andor, Choi, Collins, Mao, and Zhong]{MohriAndorChoiCollinsMaoZhong2024learning}
C.~Mohri, D.~Andor, E.~Choi, M.~Collins, A.~Mao, and Y.~Zhong.
\newblock Learning to reject with a fixed predictor: Application to decontextualization.
\newblock In \emph{International Conference on Learning Representations}, 2024.

\bibitem[Mohri and Zhong(2026{\natexlab{a}})]{MohriZhong2026slin}
M.~Mohri and Y.~Zhong.
\newblock Linear-core surrogates: Smooth loss functions with linear rates for classification and structured prediction.
\newblock In \emph{International Conference on Machine Learning}, 2026{\natexlab{a}}.

\bibitem[Mohri and Zhong(2026{\natexlab{b}})]{mohri2025beyond}
M.~Mohri and Y.~Zhong.
\newblock Beyond {T}sybakov: Model margin noise and {H}-consistency bounds.
\newblock In \emph{International Symposium on Artificial Intelligence and Mathematics}, 2026{\natexlab{b}}.

\bibitem[Mohri and Zhong(2026{\natexlab{c}})]{mohri2026principled}
M.~Mohri and Y.~Zhong.
\newblock Principled algorithms for optimizing generalized metrics in multi-label learning.
\newblock \emph{arXiv preprint arXiv:2605.28767}, 2026{\natexlab{c}}.

\bibitem[Mohri et~al.(2026)Mohri, Schneider, and Zhong]{mohri2026generalized}
M.~Mohri, J.~Schneider, and Y.~Zhong.
\newblock Generalized distributional alignment games for unbiased answer-level fine-tuning.
\newblock \emph{arXiv preprint arXiv:2605.02435}, 2026.

\bibitem[Montreuil et~al.(2025{\natexlab{a}})Montreuil, Carlier, Ng, and Ooi]{montreuil2025adversarial}
Y.~Montreuil, A.~Carlier, L.~X. Ng, and W.~T. Ooi.
\newblock Adversarial robustness in two-stage learning-to-defer: Algorithms and guarantees.
\newblock In \emph{International Conference on Machine Learning}, 2025{\natexlab{a}}.

\bibitem[Montreuil et~al.(2025{\natexlab{b}})Montreuil, Yeo, Carlier, Ng, and Ooi]{montreuil2025two}
Y.~Montreuil, S.~H. Yeo, A.~Carlier, L.~X. Ng, and W.~T. Ooi.
\newblock A two-stage learning-to-defer approach for multi-task learning.
\newblock In \emph{International Conference on Machine Learning}, 2025{\natexlab{b}}.

\bibitem[Montreuil et~al.(2026{\natexlab{a}})Montreuil, Carlier, Ng, and Ooi]{montreuil2026beyond}
Y.~Montreuil, A.~Carlier, L.~X. Ng, and W.~T. Ooi.
\newblock Beyond augmented-action surrogates for multi-expert learning-to-defer.
\newblock \emph{arXiv preprint arXiv:2604.09414}, 2026{\natexlab{a}}.

\bibitem[Montreuil et~al.(2026{\natexlab{b}})Montreuil, Carlier, Ng, and Ooi]{montreuil2026why}
Y.~Montreuil, A.~Carlier, L.~X. Ng, and W.~T. Ooi.
\newblock Why ask one when you can ask $ k $? learning-to-defer to the top-$ k $ experts.
\newblock In \emph{International Conference on Learning Representations}, 2026{\natexlab{b}}.

\bibitem[Montreuil et~al.(2026{\natexlab{c}})Montreuil, Dang, Meyer, Ng, Carlier, and Ooi]{montreuil2026online}
Y.~Montreuil, H.~D. Dang, M.~Meyer, L.~X. Ng, A.~Carlier, and W.~T. Ooi.
\newblock Online learning-to-defer with varying experts.
\newblock In \emph{International Conference on Artificial Intelligence and Statistics}, 2026{\natexlab{c}}.

\bibitem[Montreuil et~al.(2026{\natexlab{d}})Montreuil, Heng, Carlier, Ng, and Ooi]{montreuil2026optimal}
Y.~Montreuil, Y.~S. Heng, A.~Carlier, L.~X. Ng, and W.~T. Ooi.
\newblock Optimal query allocation in extractive {QA} with {LLM}s: A learning-to-defer framework with theoretical guarantees.
\newblock In \emph{International Conference on Artificial Intelligence and Statistics}, 2026{\natexlab{d}}.

\bibitem[Montreuil et~al.(2026{\natexlab{e}})Montreuil, Letian, Carlier, Ng, and Ooi]{montreuil2026adversarial}
Y.~Montreuil, Y.~Letian, A.~Carlier, L.~X. Ng, and W.~T. Ooi.
\newblock Adversarial robustness in one-stage learning-to-defer.
\newblock In \emph{International Conference on Artificial Intelligence and Statistics}, 2026{\natexlab{e}}.

\bibitem[Montreuil et~al.(2026{\natexlab{f}})Montreuil, Montreuil, Carlier, Ng, and Ooi]{montreuil2026learning}
Y.~Montreuil, L.~Montreuil, A.~Carlier, L.~X. Ng, and W.~T. Ooi.
\newblock Learning-to-defer with expert-conditioned advice.
\newblock \emph{arXiv preprint arXiv:2603.14324}, 2026{\natexlab{f}}.

\bibitem[Montreuil et~al.(2026{\natexlab{g}})Montreuil, Yu, Carlier, Ng, and Ooi]{montreuil2026time}
Y.~Montreuil, L.~Yu, A.~Carlier, L.~X. Ng, and W.~T. Ooi.
\newblock Learning to defer in non-stationary time series via switching state-space models.
\newblock \emph{arXiv preprint arXiv:2601.22538}, 2026{\natexlab{g}}.

\bibitem[Munos et~al.(2024)Munos, Valko, Calandriello, Azar, Rowland, Guo, Tang, Geist, Mesnard, Fiegel, et~al.]{munos2023nash}
R.~Munos, M.~Valko, D.~Calandriello, M.~G. Azar, M.~Rowland, Z.~D. Guo, Y.~Tang, M.~Geist, T.~Mesnard, C.~Fiegel, et~al.
\newblock Nash learning from human feedback.
\newblock In \emph{International Conference on Machine Learning}, 2024.

\bibitem[Rafailov et~al.(2023)Rafailov, Sharma, Mitchell, Manning, Ermon, and Finn]{rafailov2023direct}
R.~Rafailov, A.~Sharma, E.~Mitchell, C.~D. Manning, S.~Ermon, and C.~Finn.
\newblock Direct preference optimization: Your language model is secretly a reward model.
\newblock In \emph{Advances in Neural Information Processing Systems}, pages 53728--53741, 2023.

\bibitem[Schulman et~al.(2017)Schulman, Wolski, Dhariwal, Radford, and Klimov]{schulman2017proximal}
J.~Schulman, F.~Wolski, P.~Dhariwal, A.~Radford, and O.~Klimov.
\newblock Proximal policy optimization algorithms.
\newblock In \emph{arXiv preprint arXiv:1707.06347}, 2017.

\bibitem[Soudry et~al.(2018)Soudry, Hoffer, Nacson, Gunasekar, and Srebro]{soudry2018implicit}
D.~Soudry, E.~Hoffer, M.~S. Nacson, S.~Gunasekar, and N.~Srebro.
\newblock The implicit bias of gradient descent on separable data.
\newblock \emph{Journal of Machine Learning Research}, 19\penalty0 (70):\penalty0 1--57, 2018.

\bibitem[Steinwart(2007)]{steinwart2007compare}
I.~Steinwart.
\newblock How to compare different loss functions and their risks.
\newblock \emph{Constructive Approximation}, 26\penalty0 (2):\penalty0 225--287, 2007.

\bibitem[Stiennon et~al.(2020)Stiennon, Ouyang, Wu, Ziegler, Lowe, Voss, Radford, Amodei, and Christiano]{stiennon2020learning}
N.~Stiennon, L.~Ouyang, J.~Wu, D.~Ziegler, R.~Lowe, C.~Voss, A.~Radford, D.~Amodei, and P.~F. Christiano.
\newblock Learning to summarize with human feedback.
\newblock In \emph{Advances in Neural Information Processing Systems}, pages 3008--3021, 2020.

\bibitem[Tsochantaridis et~al.(2005)Tsochantaridis, Joachims, Hofmann, Altun, and Singer]{tsochantaridis2005large}
I.~Tsochantaridis, T.~Joachims, T.~Hofmann, Y.~Altun, and Y.~Singer.
\newblock Large margin methods for structured and interdependent output variables.
\newblock \emph{Journal of Machine Learning Research}, 6:\penalty0 1453--1484, 2005.

\bibitem[von Werra et~al.(2020)von Werra, Belkada, Tunstall, Beeching, Thrush, Lambert, Huang, Rasul, and Gallouédec]{vonwerra2022trl}
L.~von Werra, Y.~Belkada, L.~Tunstall, E.~Beeching, T.~Thrush, N.~Lambert, S.~Huang, K.~Rasul, and Q.~Gallouédec.
\newblock Trl: Transformer reinforcement learning.
\newblock \url{https://github.com/huggingface/trl}, 2020.

\bibitem[Xiao et~al.(2023)Xiao, Liu, Zhang, Muennighoff, Lian, and yun Nie]{xiao2024c}
S.~Xiao, Z.~Liu, P.~Zhang, N.~Muennighoff, D.~Lian, and J.~yun Nie.
\newblock C-pack: Packed resources for general chinese embeddings.
\newblock \emph{Proceedings of the 47th International ACM SIGIR Conference on Research and Development in Information Retrieval}, 2023.

\bibitem[Xiong et~al.(2024)Xiong, Dong, Ye, Wang, Zhong, Ji, Jiang, and Zhang]{xiong2024iterative}
W.~Xiong, H.~Dong, C.~Ye, Z.~Wang, H.~Zhong, H.~Ji, N.~Jiang, and T.~Zhang.
\newblock Iterative preference learning from human feedback: Bridging theory and practice for {RLHF} under {KL}-constraint.
\newblock In \emph{International Conference on Machine Learning}, pages 54715--54754, 2024.

\bibitem[Yang et~al.(2024)Yang, Yang, Zhang, Hui, Zheng, Yu, Li, Liu, Huang, Wei, Lin, Yang, Tu, Zhang, Yang, Yang, Zhou, Lin, Dang, Lu, Bao, Yang, Yu, Li, Xue, Zhang, Zhu, Men, Lin, Li, Xia, Ren, Ren, Fan, Su, Zhang, Wan, Liu, Cui, Zhang, and Qiu]{qwen2.5}
A.~Yang, B.~Yang, B.~Zhang, B.~Hui, B.~Zheng, B.~Yu, C.~Li, D.~Liu, F.~Huang, H.~Wei, H.~Lin, J.~Yang, J.~Tu, J.~Zhang, J.~Yang, J.~Yang, J.~Zhou, J.~Lin, K.~Dang, K.~Lu, K.~Bao, K.~Yang, L.~Yu, M.~Li, M.~Xue, P.~Zhang, Q.~Zhu, R.~Men, R.~Lin, T.~Li, T.~Xia, X.~Ren, X.~Ren, Y.~Fan, Y.~Su, Y.~Zhang, Y.~Wan, Y.~Liu, Z.~Cui, Z.~Zhang, and Z.~Qiu.
\newblock {Qwen2.5} technical report.
\newblock \emph{arXiv preprint arXiv:2412.15115}, 2024.

\bibitem[Yuan et~al.(2023)Yuan, Yuan, Tan, Wang, Huang, and Huang]{yuan2023rrhf}
H.~Yuan, Z.~Yuan, C.~Tan, W.~Wang, S.~Huang, and F.~Huang.
\newblock {RRHF}: Rank responses to align language models with human feedback.
\newblock In \emph{Advances in Neural Information Processing Systems}, 2023.

\bibitem[Zadrozny et~al.(2003)Zadrozny, Langford, and Abe]{zadrozny2003cost}
B.~Zadrozny, J.~Langford, and N.~Abe.
\newblock Cost-sensitive learning by cost-proportionate weighting of examples.
\newblock In \emph{Third IEEE International Conference on Data Mining}, pages 435--442, 2003.

\bibitem[Zhang(2004)]{Zhang2003}
T.~Zhang.
\newblock Statistical behavior and consistency of classification methods based on convex risk minimization.
\newblock \emph{The Annals of Statistics}, 32\penalty0 (1):\penalty0 56--85, 2004.

\bibitem[Zhao et~al.(2023)Zhao, Joshi, Liu, Khalman, Saleh, and Liu]{zhao2023slic}
Y.~Zhao, R.~Joshi, T.~Liu, M.~Khalman, M.~Saleh, and P.~J. Liu.
\newblock {SLiC-HF}: Sequence likelihood calibration with human feedback.
\newblock In \emph{arXiv preprint arXiv:2305.10425}, 2023.

\bibitem[Zheng et~al.(2023)Zheng, Chiang, Sheng, Zhuang, Wu, Zhuang, Lin, Li, Li, Xing, et~al.]{zheng2023judging}
L.~Zheng, W.-L. Chiang, Y.~Sheng, S.~Zhuang, Z.~Wu, Y.~Zhuang, Z.~Lin, Z.~Li, D.~Li, E.~Xing, et~al.
\newblock Judging {LLM}-as-a-judge with {MT-Bench} and {Chatbot Arena}.
\newblock In \emph{Advances in Neural Information Processing Systems}, pages 46595--46623, 2023.

\bibitem[Zhong(2025)]{zhong2025fundamental}
Y.~Zhong.
\newblock \emph{Fundamental Novel Consistency Theory: {H}-Consistency Bounds}.
\newblock PhD thesis, New York University, 2025.

\bibitem[Zhou et~al.(2023)Zhou, Liu, Shao, Yue, Yang, Ouyang, and Qiao]{wang2023beyond}
Z.~Zhou, J.~Liu, J.~Shao, X.~Yue, C.~Yang, W.~Ouyang, and Y.~Qiao.
\newblock Beyond one-preference-fits-all alignment: Multi-objective direct preference optimization.
\newblock In \emph{Annual Meeting of the Association for Computational Linguistics}, 2023.

\end{thebibliography}

\newpage
\appendix

\section{Experimental Details}
\label{app:add_experiments}

We provide the comprehensive configuration details for the experiments presented
in Section~\ref{sec:experiments}.

\paragraph{Controlled validation settings.}
For the \emph{Synonym Stress Test}, we fine-tuned Llama-3-8B using LoRA with
rank $r=16$, alpha $\alpha=16$, and a learning rate of $5\times 10^{-5}$. For
SA-DPO, the structure-aware margin scaling factor was set to $\tau=5.0$.  For
the \emph{Margin-Capacity Analysis}, we used the Anthropic HH-RLHF dataset
\citep{bai2022training}. We used a constrained LoRA configuration
($r=8, \alpha=16$) to limit model capacity. Models were trained for an extended
duration of 550 steps with a learning rate of $2\times 10^{-5}$ to ensure
convergence behavior was strictly due to loss geometry rather than insufficient
training.

\paragraph{Real-world evaluation settings.}
We fine-tuned Llama-3-8B-Instruct (4-bit quantized) on a 20k subset of the
UltraFeedback dataset \citep{cui2023ultrafeedback}. We applied LoRA
($r=32, \alpha=32$) to all linear projection layers. Training was conducted for
1,200 steps with a batch size of 16 and a learning rate of $1\times 10^{-4}$,
using a cosine schedule with 5\% warmup.

\paragraph{Argilla DPO-Mix-7k settings.}
We fine-tuned Qwen2.5-7B-Instruct \citep{qwen2.5} using LoRA ($r=32, \alpha=32$)
for 1,000 steps with a batch size of 16 and a learning rate of
$5 \times 10^{-5}$. The margin configurations ($\gamma$ and $\tau$) were aligned
using the same median distance approach as the UltraFeedback experiments. The
downstream generation evaluation was conducted using 200 hold-out prompts from
MT-Bench \citep{zheng2023judging}, scored by the \texttt{PairRM} model
\citep{jiang2023llmblender}.
\begin{itemize}
\item \textbf{SimPO Instantiation:} We performed a grid search on a hold-out
  validation set to select the optimal fixed margin $\gamma=0.7$, fixing
  $\beta=1.0$.
\item \textbf{SA-DPO Instantiation:} We defined the structure-aware margin
  $\Gamma(y, y') = \tau \Delta(y, y')$, using the cosine distance $\Delta$ from
  \texttt{BGE-Large-v1.5} embeddings \citep{xiao2024c}. We calibrated
  $\tau \approx 3.4$ to align the average margin on all training pairs with
  SimPO ($\mathbb{E}[\Gamma] = \tau \mathbb{E}[\Delta] \approx 0.7$), using the
  same $\beta=1.0$.
\end{itemize}

\end{document}